\documentclass{article}
\usepackage[accepted,nohyperref]{icml2024}
\usepackage[utf8]{inputenc} 
\usepackage[T1]{fontenc}    
\usepackage{bm,bbm}       
\usepackage{url}            
\usepackage{booktabs}       
\usepackage{nicefrac}       
\usepackage{microtype}      
\usepackage{lipsum}		
\usepackage{amsfonts,amsmath,amsthm,amssymb}
\usepackage{booktabs,multirow,threeparttable}
\usepackage{graphicx}
\usepackage[caption=false]{subfig}
\usepackage[most]{tcolorbox}
\usepackage{mathtools}

\newcommand{\Tr}{{\rm Tr}}

\usepackage{natbib}
\usepackage{doi}
\usepackage{xcolor}

\usepackage[capitalise]{cleveref}
\crefname{lemma}{Lemma}{Lemma}
\crefname{definition}{Definition}{Definition}
\crefname{corollary}{Corollary}{Corollary}
\crefname{assumption}{Assumption}{Assumption}
\crefname{proposition}{Proposition}{Proposition}

\newtheorem{theorem}{Theorem}[section]
\newtheorem{corollary}{Corollary}[theorem]
\newtheorem{lemma}[theorem]{Lemma}
\newtheorem{assumption}[theorem]{Assumption}

\newtheorem{definition}{Definition}

\newcommand{\mI}{\bm{I}}

\newcommand{\veps}{\bm{\varepsilon}}

\newcommand{\vphi}{\bm{\phi}}

\newcommand{\vpsi}{\bm{\psi}}

\newcommand{\vLambda}{\bm{\Lambda}}

\newcommand{\vSigma}{\bm{\Sigma}}

\newcommand{\vf}{\bm{f}}

\newcommand{\vx}{\bm{x}}
\newcommand{\vy}{\bm{y}}
\newcommand{\vz}{\bm{z}}
\newcommand{\vA}{\bm{A}}

\newcommand{\vI}{\bm{I}}

\newcommand{\vK}{\bm{K}}

\newcommand{\vT}{\bm{T}}
\newcommand{\vU}{\bm{U}}
\newcommand{\vV}{\bm{V}}
\newcommand{\vW}{\bm{W}}
\newcommand{\vX}{\bm{X}}

\newcommand{\vZ}{\bm{Z}}
\newcommand{\cH}{\mathcal{H}}

\def\bx{\bm{x}}

\def\bB{\bm{B}}
\def\bV{\bm{V}}

\newcommand{\bR}{\mathbb{R}}
\newcommand{\bE}{\mathbb{E}}

\def\ddefloop#1{\ifx\ddefloop#1\else\ddef{#1}\expandafter\ddefloop\fi}

\def\ddef#1{\expandafter\def\csname c#1\endcsname{\ensuremath{\mathcal{#1}}}}
\ddefloop ABCDEFGHIJKLMNOPQRSTUVWXYZ\ddefloop

\def\ddef#1{\expandafter\def\csname s#1\endcsname{\ensuremath{\mathsf{#1}}}}
\ddefloop ABCDEFGHIJKLMNOPQRSTUVWXYZ\ddefloop

\def\E{\mathbb{E}}

\def\argmin{\operatornamewithlimits{arg\,min}}

\def\rd{{\mathrm d}}

\icmltitlerunning{Importance Reweighting: Data-Dependent Implicit Regularization}

\begin{document}

\twocolumn[
\icmltitle{{High-Dimensional Kernel Methods under Covariate Shift: \\
Data-Dependent Implicit Regularization}}

\icmlsetsymbol{equal}{*}

\begin{icmlauthorlist}
\icmlauthor{Yihang Chen}{sch1}
\icmlauthor{Fanghui Liu}{sch2}
\icmlauthor{Taiji Suzuki}{sch3,sch4}
\icmlauthor{Volkan Cevher}{sch1}
\end{icmlauthorlist}

\icmlaffiliation{sch1}{Laboratory for Information and Inference Systems, École Polytechnique Fédérale de Lausanne (EPFL), Switzerland.}
\icmlaffiliation{sch2}{Department of Computer Science, University of Warwick, United Kingdom.}
\icmlaffiliation{sch3}{Department of Mathematical Informatics, The University of Tokyo, Japan.}
\icmlaffiliation{sch4}{Center for Advanced Intelligence Project, RIKEN, Tokyo, Japan}
\icmlcorrespondingauthor{Fanghui Liu}{fanghui.liu@warwick.ac.uk}

\icmlkeywords{Machine Learning, ICML}
\vskip 0.3in
]

\printAffiliationsAndNotice{}  

\begin{abstract}

This paper studies kernel ridge regression in high dimensions under covariate shifts and 
analyzes the role of importance re-weighting. 
We first derive the asymptotic expansion of high dimensional kernels under covariate shifts. By a bias-variance decomposition, we theoretically demonstrate that the re-weighting strategy allows for decreasing the variance. 
For bias, we analyze the regularization of the arbitrary or well-chosen scale, showing that the bias can behave very differently under different regularization scales. 
In our analysis, the bias and variance can be characterized by the spectral decay of a data-dependent regularized kernel: the original kernel matrix associated with an additional re-weighting matrix, and thus the re-weighting strategy can be regarded as a data-dependent regularization for better understanding. Besides, our analysis provides asymptotic expansion of kernel functions/vectors under covariate shift, which has its own interest.

\end{abstract}

\section{Introduction}
In statistical learning theory \cite{vapnik1999nature}, the fundamental assumption is that the training and test data are drawn from the same distribution. 
However, in real-world applications, test data may generated quite differently from the training data.  
One of the most common situations is the \emph{covariate shifts} \citep{shimodaira2000,sugiyama2007direct}, where the training and test distributions of inputs (covariates) are different.

The {\em importance weighting (IW)} \citep{shimodaira2000} is a typical way to handle covariate shift. Let $p$ and $q$ be the marginal distributions over the training and test covariates, respectively, the IW method adopts their Radon-Nikodym derivative as the importance weighting (IW) function, i.e., ${w}(\vx) = \rd q(\vx) / \rd p(\vx)$. 
Hence, the IW function weights the loss function, leading to an unbiased estimator of the expected loss under the test distribution. 
Empirically, the IW method has been widely used in machine learning \citep[e.g.,][]{huang2006correcting,sugiyama2008direct,cortes2010,sugiyama2012density,fang2020rethinking} from linear to kernel estimator as well as neural networks.
Theoretically, the IW method can achieve nice statistical properties (e.g., minimax rate) under certain settings for kernel ridge regression \cite{ma2023optimally,gogolashvili2023importance}.

However, recent work on high-capacity models, e.g., nonparametric and over-parameterized models\footnote{Over-parameterized models admit the fact that the number of parameters is larger than the number of training data. Modern neural networks belong to this setting.}, demonstrate that, the IW strategy is not beneficial under certain settings, e.g., over-parameterized linear regression \cite{zhai2023understanding}, $k$-nearest neighbors classifier \cite{kpotufe2021marginal} for \emph{interpolation} under well-specified cases. Nevertheless, for some misspecified cases, the IW correction is still needed for non-parametric kernel ridge regression \cite{gogolashvili2023importance}.

We can see the separation in the effect of IW for low/high-capacity models under (mis)-specified settings.
But how the generalization result depends on the choice of model capacities, and its interplay with the regularization level in terms of bias-variance trade-off remains unclear. 
Intuitively, the IW strategy obtains the unbiased estimation of the original empirical risk minimization, leading to a decreasing variance to some extent;
while the approximation between the estimator and the target function will change, leading to an increasing bias to some extent.
As such, refined analyses based on bias-variance trade-offs are required to understand the following question: 
\begin{center}
\emph{How does IW affect bias-variance trade-off in high-capacity models?}
\end{center}

We attempt to address this question by uncovering the mystery behind the IW strategy in covariate shifts from the bias-variance trade-off.
To be specific, in this work, we focus on kernel ridge regression (KRR) in \emph{high dimensions} with data dimension $d$ and size $n$ both large under the IW strategy, a typical regularized-based nonparametric regression over reproducing kernel Hilbert spaces (RKHSs). 
This choice allows for studying different learning paradigms, for example, neural networks can be described by neural tangent kernel \cite{jacot2018neural} under certain settings; the high-dimensional setting matches practical image application via over-parameterized neural networks; 
the model capacity can be tuned by the regularization parameter.
Accordingly, the kernel interpolation can be regarded as a special case of KRR by taking the explicit regularization sufficiently close to zero, which follows the spirit of over-parameterized neural networks for \emph{interpolation learning}.

Formally, given $n$ training data $\vZ=\{(\vx_i,y_i) \}_{i=1}^n$, the estimator of KRR in high dimensions under a general IW function $\overline{w}(\vx)$ is given by
\begin{equation}\label{eq:iw_emp_risk_Z}
   \overline{f}_{\lambda,\vZ}\!:=\! \argmin_{f \in \mathcal{H}} \! \left\{\! \frac{1}{n} \sum_{i=1}^{n}\overline{w}(\vx_i)\left(f\left(\vx_{i}\right) \!-\!y_{i}\right)^{2} \!+\! \lambda \|f\|_{\mathcal{H}}^{2} \! \right\}\,,
\end{equation}
where $\lambda >0$ is the regularization parameter. 

\subsection{Contributions}
We summarize the contributions and findings as below:
\begin{itemize}
\item We present the asymptotic expansion of high dimensional kernels $k(\vx, \vx')$ under covariate shifts, where the nonlinearity in kernels can be eliminated by the kernel function curvature, see \cref{lemma:KXx_approx}.
    \item We present bias-variance decomposition for KRR in high dimensions with covariate shift. To be specific, for variance, via the asymptotic expansion, we demonstrate that the IW strategy can be regarded as an implicit data-dependent regularization on the respective kernel. The estimation of variance heavily depends on the spectral decay of the expected covariance matrix over $q$ or such data-dependent regularized kernel, and allows for a decreasing variance to some extent, see~\cref{sec:variance}.  
    \item For bias, via the asymptotic expansion, we demonstrate that i) near interpolation (i.e., the regularization $\lambda$ is sufficiently small), the bias term can be upper bounded by two parts, one is an intrinsic bias that only depends on the covariate shift problem itself, in a constant order; another is the importance re-weighting bias, which depends on the spectral decay of data-dependent regularized kernels. ii) if we choose a proper regularization parameter, the IW strategy does not hurt the bias, i.e., the bias can tend to zero, see~\cref{sec:bias} for details.  
\end{itemize}
We hope our analysis provides a better understanding on the role of the IW strategy in terms of bias-variance trade-off, and would like to motivate the community to think about powerful IW strategies to handle distribution shifts, more generally.

\subsection{Related works}

\paragraph{High-dimensional kernel regression} To tackle the high-dimensional regression, one line of research~\citep{mei2021learning,mei2022generalization,ghorbani2020neural,ghorbani2021linearized,misiakiewicz2022learning,xiao2022precise,ghosh2021three,fang2020rethinking,aerni2023strong} asymptotically characterizes the
precise risk of kernel regression under some specific data distributions, such as uniform distributions on the sphere or hypercube vertices, so that the kernel’s eigenfunctions and eigenvalues can be explicitly accessed. Another line of research~\citep{liang2020just,liu2021kernel,mcrae2022harmless} provides non-asymptotic bounds by high-dimensional random matrix concentration in \citet{el2010spectrum}. 

\paragraph{Covariate shift} 
There has been extensive analysis of kernel regression under covariate shift in the fixed dimensions. In the well-specified case, the standard maximum likelihood estimation leads to the optimal model, and the importance re-weighting is unnecessary~\citep{zhai2023understanding,ge2023maximum}. \citet{ma2023optimally,gogolashvili2023importance} analyze different importance re-weighting functions. \citet{feng2023towards} additionally provide a uniform analysis for kernel regression of general loss function under covariate shifts. However, the analysis of the fixed-dimension kernel requires an appropriate choice of $\lambda$ to balance the bias and variance.
Apart from re-weighting, the transfer exponent \cite{kpotufe2021marginal} is another metric to evaluate the distribution mismatch, as well as another variant \cite{pathak2022new}.

\paragraph{Random matrix theory} In the specific case of the linear kernel,  a series of works use the random kernel theory to asymptotically characterize the
precise risk~\citep{hastie2022surprises,karoui2013asymptotic,dicker2016ridge,wu2020optimal,lu2023optimal}. There is also a series of works focusing on covariate shift in the high-dimensional random feature regression~\cite{tripuraneni2021overparameterization,tripuraneni2021covariate}. 
However, their results did not consider the data-dependent importance re-weighting, explained as below.

Classical RMT is able to provide an exact characteristic formulation of the limiting distribution of covariance matrix via its Stieltjes transform, and then its solution can be obtained from the popular Mar{\u c}enko–Pastur equation.
However, since the IW strategy is regarded as a data-dependent transformation (we will discuss it later), the limiting distribution of the ``data-dependent'' covariance matrix can not be directly obtained, which requires more effort and advanced techniques in the RMT community. We leave this as an open question.

\paragraph{Notations} We denote the decreasing eigenvalues of any matrix $\vA\in \mathbb{R}^{n\times n}$ by $\lambda_1(\vA)\geq \lambda_2(\vA)\dots\geq \lambda_n(\vA)$, and the spectrum of $\vA$ by $\vLambda(\vA):=\{\lambda_i(\vA)\}_{i=1}^n$. We call $a\lesssim b$ or $a=O(b)$ if and only if there exists constant $C$ independent of $n,d$, such that $a\leq C\cdot b$. We call a positive function $f(d)\asymp d^a$ if and only if $\sup\lim_{d\to\infty} f(d)/d^{a+\epsilon}=0, \inf\lim \lim_{d\to\infty} f(d)/d^{a-\epsilon}=+\infty$ for any $\epsilon>0$. 
We use the abbreviation $[d]=\{1,2,\cdots,d\}$ for integer $d$. 

{\bf Organization} The paper is organized as below: \cref{sec:settings} introduce our problem settings and \cref{sec:assum} makes the required assumptions for our proof. Our main results are given in \cref{sec:kernel_approx} and the conclusion is drawn in \cref{sec:conclusion}.

\section{Problem Settings}
\label{sec:settings}

We introduce our problem settings in terms of the data generation process under covariate shift and the used kernel function in RKHS.

{\bf Data generation process:} We follow the classical statistical learning framework \cite{cucker2007learning}.
Let $\mathcal{X} \subseteq \mathbb{R}^d$ be the input space (compact domain) and $\mathcal{Y}\subset \mathbb{R}$ is the label space, we observe $n$ i.i.d. pairs $\vZ=\{(\vx_i,y_i)$, $1\leq i\leq n\}$, where $\vx_i$ are the covariates and $y_i\in\mathcal{Y}$ are the labels. 
 Suppose these $n$ pairs are drawn from a unknown probability distribution $p(\vx,y) := p(\vx)\rho(y|\vx)$, where $p(\vx)$ as the marginal distribution of $\rho$ on $\mathcal{X}$ and $\rho(y|\bm x)$ as the conditional distribution at $\bm x\in\mathcal{X}$ induced by $\rho$.
Let $\widehat{\E}_n$ be the expectation on the empirical measure $\widehat{p}_n(\vx):=\frac{1}{n}\sum_{i=1}^n \delta_{\vx_i}(\vx)$. 
 The objective of our learning problem is to find a learning model that is a good approximation of the ``target function''
  $ f_\rho(\vx) = \int_{Y} y \mathrm{d} \rho(y|\vx)\,, \forall \vx \in X $
 as the conditional mean.
 We assume that there exists a $\sigma_\varepsilon>0$ such that $y(\vx) = f_\rho(\vx)+\varepsilon$, and $\E [\varepsilon] = 0$, $\mathbb{V} [\varepsilon] \leq \sigma_\varepsilon^2$. 

{\bf Re-weighting in covariate shift:} Under the covariate shift setting where the test data is not sampled from $p(\vx)$ but the test distribution as $q(\vx)$.
To handle this, we introduce the importance re-weighting strategy with the density ratio ${w}(\vx) = \rd q(\vx) / \rd p(\vx)$.
Here we consider a general version by introducing the weighting distribution $\overline{q}(\vx)$ such that $\overline{w}(\vx):= \rd \overline{q}(\vx) / \rd p(\vx)$, where we use $\overline{w}(\vx)$ as importance weighting. In general, $\overline{q}$ can be unnormalized density, with $\overline{Z}:=\int_{\vx\sim \vX} \rd \overline{q}(\vx)$. However, without loss of generality, we can assume $\overline{Z}=1$. Otherwise, we can replace $\lambda$ with $\lambda \overline{Z}$. 
Accordingly, when $\overline{w}(\vx):=1$, our minimization problem is reduced to the standard unweighted empirical risk minimization; when $\overline{w}(\vx):= w(\vx)$, it is reduced to the standard importance re-weighting by the density ratio. 

In this paper, the used learning model is kernel ridge regression endowed by RKHS in high dimensions as described below, where the training dataset size $n$ and data dimension $d$ satisfy $n/ d\to \zeta$ with $\zeta \in (0,\infty)$ as $d\to \infty$, and $\zeta_{\min}\leq n/d\leq \zeta_{\max},\forall n,d$. This is the standard setting in high-dimensional kernel regression ~\cite{liang2020just,liu2021kernel,mei2022generalization}.

\subsection{RKHS and kernels}
\label{sec:formulation}

The Reproducing kernel Hilbert space (RKHS) $\mathcal{H}$ is a Hilbert space $\mathcal{H}$ endowed with the inner product $\langle \cdot,\cdot\rangle_K$ of functions $f: \mathcal{X} \rightarrow \mathbb{R}$ with a reproducing kernel $K: \mathcal{X} \times \mathcal{X} \rightarrow \mathbb{R}$ where $K(\cdot) \in \mathcal{H}$ and $f(\vx) = \langle f, K(\vx, \cdot) \rangle_K$~\citep{mercer1909xvi}. We assume that $K$ is bounded, i.e., there exists a constant $1 \leq \kappa < \infty$ such that $\sup_{\vx\sim \mathcal{\vX}} K(\vx,\vx)\leq \kappa$. 

Define $\mathcal{L}_q^2:=\{f:\mathcal{X}\to\mathbb{R}| \|f\|_q^2\leq \infty\}$, and $\|f\|_q^2:= \int_{\mathcal{X}} f^2(\vx) \rd q(\vx) $. 
For ease of our analysis, let us introduce the integral operator ${L}_q: \mathcal{L}_q^2 \rightarrow \mathcal{L}_q^2$ with respect to the test distribution $q(\vx)$:
\begin{align*}
	L_q f = \int K(\cdot, \vx^\prime) f(\vx^\prime) \rd q(\vx^\prime)\,,
\end{align*}
and denote the set of eigenfunctions of this integral operator by $\vphi(\vx) = \{\phi_1(\vx), \phi_2(\vx), \ldots, \phi_o(\vx)\}$, where $o$ could be $\infty$. We have that
\begin{align}
	L_q \phi_i = \lambda_i \phi_i,~~\text{and}~~ \int \phi_i(\vx) \phi_j(\vx) \rd q(\vx) = \delta_{ij}\,.
\end{align}
Denote $\vLambda = {\rm diag}(\lambda_1,\cdots, \lambda_o)$ as the collection of non-negative eigenvalues, with $\lambda_1\geq \lambda_2\geq\cdots\geq  \lambda_o$. We can write $K(\cdot,\cdot)$ via the spectral notation  
\begin{align*}
	K(\vx, \vx^\prime) = \vphi(\vx)^\top \vLambda \vphi(\vx^\prime)\,. 
\end{align*}
We define the empirical integral operator on the training dataset $\vX$, 
\begin{align*}
   L_{{q},\vX} f := \frac{1}{n} \sum_{i=1}^n {w}(\vx_i) f(\vx_i) K(\cdot,\vx_i)\,.  
\end{align*}

Similarly, we can define $L_{\overline{q}}$, and $L_{\overline{q},\vX}$ for weighting distribution $\overline{q}$.

\subsection{Interpolation and regression}

Let $\vX := [\vx_1, \vx_2, \cdots, \vx_n]^\top \in \mathbb{R}^{n \times d}$ be the data matrix,  $\vy := [y_1, y_2, \cdots, y_n]^\top\in\bR^n$ be the label vector, and $\vZ:=[\vX,\vy]$ be the concatenation. 
Besides, we denote $\vK(\vX, \vX) = [K(\vx_i, \vx_j)]_{ij} \in \mathbb{R}^{n \times n}$ be the kernel matrix. Extending this definition, for $\vx \in \mathcal{\vX}$ we denote by $\vK(\vx, \vX) \in \mathbb{R}^{1 \times n}$ the matrix of values $[\vK(\vx,\vx_1),\ldots,\vK(\vx,\vx_n)]$, and $\vK(\vX,\vx):=\vK(\vx,\vX)^\top\in\bR^{n\times 1}$.
 
\paragraph{Interpolation}
The unweighted interpolation estimator is defined as 
\begin{align}
	\label{eq:interpolation}
	f_{\vZ} := \argmin_{f \in \mathcal{H}} \| f \|_K, ~~ \text{s.t.}~~ f(\vx_i) = y_i,~\forall i \in [n] \,.
\end{align}
When $\vK(\vX,\vX)$ is invertible\footnote{For ease of analysis, we assume the $\vK(\vX,\vX)$ has full rank.}, solution to \eqref{eq:interpolation} can be written in the closed form:
\begin{align}
	\label{eq:interpolation_closedform}
	f_{\vZ}(\vx) &= \vK(\vx, \vX) \vK(\vX, \vX)^{-1} \vy \,.
\end{align}
Actually, in the interpolation problem, the IW strategy does work due to the constraint $\bar{w}_i [f(\bm x_i) - y_i]=0$, which naturally coincides with \citet{zhai2023understanding}.
Accordingly, we consider the regularized regression weighted by $\overline{w}(\vx)$ in \cref{eq:iw_emp_risk_Z}. Let ${ \overline{\vW}({\vX})}:={\rm diag}(\overline{w}(\vx_i))_{i=1}^n$, the solution to \cref{eq:iw_emp_risk_Z} can be written in the closed form:
\begin{align*}
	\overline{f}_{\lambda,\vZ}(\vx)& = \vK(\vx, \vX) (\vK(\vX, \vX)+\lambda n{ \overline{\vW}({\vX})}^{-1})^{-1} \vy\,.
\end{align*}

We are interested in the generalization performance of $\overline{f}_{\lambda,\vZ}$ estimated by the excess risk w.r.t. the test distribution $q$
 $\|\overline{f}_{\lambda,\vZ}-{f}_\rho\|_q^2$.

\section{Assumptions}
\label{sec:assum}

In this paper, we make the following assumptions, including the type of the considered kernels, data, and ratio.
Besides, we also introduce assumptions on the model, e.g., the source condition on the target function, and the capacity condition. 

\subsection{Basic assumptions on kernel, data distribution}
Firstly, we consider the two forms of kernels in this paper for asymptotic expansion:
\begin{itemize}
    \item \emph{inner product kernel}, $K(\vx, \vx') := h\left( \langle \vx, \vx'\rangle/d \right)$;
     \item \emph{radial kernel}, $K(\vx,\vx'):=h(-\|\vx-\vx'\|_2^2/d)$.
\end{itemize}
where $h(\cdot): \mathbb{R} \rightarrow \mathbb{R}$ is a non-linear Lipschitz smooth function in a neighborhood of $0$. 
Following \cite{el2010spectrum}, we assume $h$ to ensure the positive definiteness of the asymptotic expansion of the original kernel.

\begin{assumption}[Assumptions on $h$]\label{assumption:h}
    We assume $h:\mathbb{R}\to\mathbb{R}$ is a smooth function that satisfies the following constraints in the neighborhood of 0,
    \begin{align*}
       h(x)\geq 0, h^\prime(x)> 0,~~ h^{\prime\prime}(x)> 0,~~ h^{\prime\prime} (x)\leq M_h\,. 
    \end{align*}
\end{assumption}

{\bf Remark}: We give an example of three widely-used kernels and their corresponding non-linear activation $h$. Each 
instantiation of $h$ satisfies \cref{assumption:h}. 
\begin{table}[!htbp]
    \centering
    \caption{Kernels and their corresponding $h$.}
    \begin{tabular}{ccc}
    \toprule
       Kernel  & Formulation & $h(x)$  \\
       \midrule
      Polynomial   & $(1+\frac{1}{d}\langle\vx,\vx^\prime\rangle)^k$ & $(1+x)^k$\\
      Exponential & $\exp(\frac{2}{d}\langle\vx,\vx^\prime\rangle)$ & $\exp(2x)$\\
      Gaussian & $\exp(-\frac{1}{d}\|\vx-\vx^\prime\|_2^2)$ & $\exp(x)$\\
      \bottomrule
    \end{tabular}
    \label{tab:kernel_h}
\end{table}

In the next, we consider a general class of data distributions of $\vx\in\bR^d$. 
\begin{definition}\label{def:distribution}
Denote $\mathcal{P}_0$ as the set of distributions of the random variable $\vx\sim \mu$ satisfying the following properties. 

We assume there exists $\vSigma_\mu\in\bR^{d\times d}$, such that $\vz=\vSigma_\mu^{-1/2}\vx\in\mathbb{R}^d$. Each element of $\vz$ is independent and identically distributed on some distribution ${\widetilde{\mu}}$. We make the following assumptions on $\widetilde{\mu}$, 
 \begin{itemize}
 \item {\bf Sub-Gaussian.} $\widetilde{\mu}$ is sub-Gaussian. 
    \item {\bf Identity Variance.} Define the $i$-th moment of distribution $\widetilde{\mu}$, $\kappa_{{\mu},i}:=\E_{z \sim {\widetilde{\mu}}} (z)^i$, we have $\kappa_{\widetilde{\mu}, 1}=0, \kappa_{\widetilde{\mu},2}=1$, i.e., $\bE_{\vz} \vz\vz^\top = \vI$. 
    \item {\bf Uniform Boundedness.} There exists integer $m_\mu\geq 0$, such that $|{\vz}(k)|\lesssim d^{\frac{2}{8+m_r}}$. We additionally define constant $\theta_\mu:=\frac{1}{2}-\frac{2}{8+m_\mu}$ for future simplicity. 
\end{itemize}
\end{definition}

\begin{assumption}[Bounded Distribution]\label{assumption:8+m}
The training distribution $p$ and test distribution $q$ belong to $\mathcal{P}_0$, with $\vSigma_p,\vSigma_q,m_p,m_q,\theta_p,\theta_q,\tau_p,\tau_q$ being defined in \cref{def:distribution}. 
\end{assumption}
{\bf Remark:} This assumption (or distribution class $\mathcal{P}_0$) is widely used in high-dimensional statistics \cite{liang2020just,liu2021kernel,wu2020optimal}.
The data distribution is assumed to be not too heavy-tailed, with possible structure between the entries with zero-mean and unit-variance and bounded moment with respect to $d$. The identity variance assumption ensures that $\bE_{\vx\sim \mu}\vx={\bf 0}, \bE_{\vx\sim \mu} \vx\vx^\top =\vSigma_\mu$, i.e., $\vSigma_\mu$ is the covariance matrix of $\vx\sim\mu$.

\begin{assumption}[Similar Covariate]\label{assumption:bounded_shift}
We assume $\max\{\|\vSigma_p\|,\|\vSigma_q\|\}=O(1)$. 
Define $\vSigma_{pq}:=\vSigma_p^{-1}\vSigma_q$, and $\exists c_{pq} \geq 0$ such that
$\Tr(\vSigma_{pq})/d\lesssim d^{c_{pq}}$. To bound the distribution shifts, we additionally assume $c_{pq}<2\theta_q-\frac{1}{2}=\frac{1}{2}-\frac{4}{8+m_q}$.
\end{assumption}
{\bf Remark:} 
When $\vSigma_p=\vSigma_q$, we have $c_{pq}=0$, this assumption always holds due to $m_q>0$. We make this assumption to provide a more precise characterization of the similarity between $\vSigma_p$ and $\vSigma_q$ via $\langle \vSigma_p^{-1}, \vSigma_q \rangle$, which aims to describe the difficulty of distribution shift. The distribution shift is small when $c_{pq}$ is close to 0. The upper bound for $c_{pq}$ is a sufficient condition to ensure the linear approximation of the kernel $K$, see \cref{lemma:KXx_approx}.

For the ratios $w(\vx),\overline{w}(\vx)$, we make the following assumption.
\begin{assumption}[Bounded Ratio \cite{gogolashvili2023importance}]\label{assumption:IW_assumption}
For the probability ratio $v\in\{w,\overline{w}\}$, there exist constants $t_v \in [0,1]$, $W_v(d) > 0$ and $\sigma_v(d) > 0$, where $W_v(d),\sigma_v(d)$ is dependent on dimension $d$, such that, for all $m \in \mathbb{N}$ with $m \geq 2$, it holds that  
\begin{align}\label{condition_on_iw}
  &  \left(\int_X v(x)^{\frac{m-1}{t_v}} \rd q(x)\right)^{t_v} \leq \frac{1}{2}m!W_v(d)^{m-2}\sigma_v(d)^2\,,
\end{align}
where the left-hand side for $t_v = 0$ is defined as $\left\| v^{m-1} \right\|_\infty$, the essential supremum of $v^{m-1}$ with respect to $q$. We additionally assume that for sufficiently large $d$, 
\begin{align*}
    W_v(d) \leq W_v \cdot d^{c_{v,1}},  \sigma_v(d) \leq \sigma_v \cdot d^{c_{v,2}}\,,
\end{align*}
with $c_{v,1}\leq 2c_{v,2}$, and $c_{v,2}\leq \frac{1}{4}$. 
\end{assumption}
{\bf Remark}: 
\cref{assumption:IW_assumption} covers the uniform bounded ratio by taking $t_w=0, W_w=\sigma_w^2=\arg\max_{\vx} w(\vx)$. 

\subsection{Assumptions on model}

In the next, we present the used assumptions for our analysis of the target function and model capacity. 
Firstly, we consider the source condition of the target function $f_\rho$. 
\begin{assumption}\label{assumption:sourcecon} (Source condition \cite{smale2004shannon,smale2007learning})
We have $f_\rho\in\mathcal{H}$, and there exists $\frac{1}{2} \leq \overline{r} < 1, \overline{g}_{\rho} \in \mathcal{L}_q^2$ such that ${f}_{\rho} = (L_{\overline{q}})^{\overline{r}} \overline{g}_{\rho}$. We additionally assume $\max\{\|f_\rho\|_{\mathcal{H}}, \|\overline{g}_\rho\|_q,\|f_\rho\|_\infty\}\leq C_{\mathcal{H}} d^{c_{\mathcal{H}}}$. 
\end{assumption}
{\bf Remark}: Source condition is widely used in the kernel literature~\cite{smale2004shannon,smale2007learning,caponnetto2007optimal}. Intuitively, a larger $\overline{r}$ indicates that $f_{\rho}$ is smoother. When $q=\overline{q}$, \cref{assumption:sourcecon} is reduced to the standard source condition on distribution $q$. When $\overline{r}=1/2$, we have $\|f_\rho\|_{\mathcal{H}}=\|g_{\rho}\|_q$. One key difference with classical high-dimensional analysis~\citep{liang2020just} is that we do not always need a uniform constant upper bound of $\|f_\rho\|_{\mathcal{H}}$ over $d$.

For a kernel matrix $\vK$, We define it capacity by $\mathcal{N}(\vK, b)$, which is widely used in \cite{nakkiran2020optimal,dobriban2018high,liang2020just,jacot2020kernel,nakkiran2020optimal}.
\begin{definition}[Capacity]\label{def:capacity_k}
Given a kernel matrix $\vK$ and a parameter $b >0$, we denote its capacity as
\begin{align*}
    \mathcal{N}(\vK,b):= \Tr\left[(\vK+b\vI)^{-2}\vK\right] = \sum_{i=1}^n \frac{\lambda_i(\vK)}{\left(b+ \lambda_i(\vK)\right)^2}\,.
\end{align*}
\end{definition}
The capacity can also be defined for the operator. The following assumption describes the model capacity of kernel methods in terms of "effective dimension". 
\begin{assumption}[Capacity condition~\citep{caponnetto2007optimal}]\label{assumption:effective_dimension}
    For any $\lambda>0$, there exists $E_\mu>0$ and $s_\mu\in[0,1]$ such that for distribution $\mu\in \{q,\overline{q}\}$,
    \begin{align*}
        \mathcal{N}_\mu(\lambda) := \Tr((L_\mu+\lambda)^{-1}L_\mu) \leq E_\mu^2 \lambda^{-s_\mu},\forall \lambda \in (0,1]\,. 
    \end{align*}
\end{assumption}
{\bf Remark}: The effective dimension $\mathcal{N}_{\mu}(\lambda)$ measures the capacity of the kernel regression model with the regularization $\lambda$, which can be interpreted by an estimate of the number of eigenvalues of $L_r$ larger than $\lambda$. If the eigenvalues of $L_r$, i.e. $\lambda_{r,i}$,  decay at the asymptotic order $O(i^{-1/s_\mu})$, \cref{assumption:effective_dimension} holds. A small $s_\mu$ indicates that the eigenvalues of $L_r$ decay at a faster rate, and \cref{assumption:effective_dimension} always holds when $s_\mu=1$ and $E_\mu = \sqrt{\kappa}$, where $\kappa=\max \{\sup_{\vx\in\mathcal{X}}K(\vx,\vx),1\}$. 

\subsection{Summary of notations}
We have introduced several constants in this assumption above. We summarize it here. 

{\bf $\vSigma_\mu, \widetilde{\mu}, \kappa_{\mu, i}, m_\mu, \theta_\mu$ }: assumptions on distribution $\mu=p,q$. See \cref{def:distribution}. 

{\bf $c_{pq}$}: trace of $\vSigma_{pq}$. See \cref{assumption:bounded_shift}. 

{\bf $t_v, W_v(d), \sigma_v(d), c_{v,1}, c_{v,2}$}; upper bound of the probability ratio. See \cref{assumption:IW_assumption}. 

{\bf $\overline{r}, c_{\mathcal{H}}$}: source condition. See \cref{assumption:sourcecon}. 

{\bf $s_\mu, E_\mu$}: effective dimension. See \cref{assumption:effective_dimension}.

\section{Main results}\label{sec:kernel_approx}
In this section, we present the main results: the bias and variance the excess risk of the estimator $\overline{f}_{\lambda,\vZ}$ can be conducted from bias-variance decomposition.
Then we derive the estimation for the bias and variance, respectively.

\subsection{Bias-variance decomposition}

To conduct bias-variance decomposition, we need the noiseless version of \cref{eq:iw_emp_risk_Z} for analysis.
\begin{equation}
      \label{eq:iw_emp_risk_X}
\begin{split}
    &\overline{f}_{\lambda,\vX}:=\\
    &\argmin_{f \in \mathcal{H}}
   \left\{\frac{1}{n} \sum_{i=1}^{n}\overline{w}(\vx_i)\left(f\left(\vx_{i}\right) \!-\!f_\rho(\vx_i)\right)^{2} \!+\! {\lambda}\|f\|_{\mathcal{H}}^{2}\right\}\,,
\end{split}
\end{equation}
i.e., to replace $y_i$ with its expectation $f_\rho(\vx_i)$. Using the notations of the empirical operator, we have
\begin{align*}
    \overline{f}_{\lambda,\vX} = (L_{\overline{q},\vX}+{\lambda} I)^{-1} L_{\overline{q},\vX} f_\rho\,. 
\end{align*}

We then provide the bias-variance decomposition $\|\overline{f}_{\lambda,\vZ}-{f}_\rho\|_q$ by the following lemma, with the proof deferred to \cref{app:biasvariance}.
 \begin{lemma}
	\label{lemma:decomposition}
 We consider the excess risk $\|\overline{f}_{\lambda,\vZ}-{f}_\rho\|_q$ conditioned on $\vX$ for our re-weighting estimator \eqref{eq:iw_emp_risk_Z}, admitting the following bias-variance decomposition:
	\begin{align*}
 		&\E_{\vy|\vX} \| \overline{f}_{\lambda,\vZ} - {f}_\rho\|^2_q\\
   =&{\E_{\vy|\vX} \|\overline{f}_{\lambda,\vZ}-\overline{f}_{\lambda,\vX}\|_{q}^2} + {\|\overline{f}_{\lambda,\vX}-{f}_\rho\|_{q}^2} := {\sf V}+{\sf B}^2 \,.
   \end{align*}
\end{lemma}
Clearly, the bias term does not rely on the label noise and the variance is independent of the target function $f_{\rho}$, which matches the spirit of the bias-variance decomposition.

\subsection{Asymptotic expansion of high dimensional kernels}
Considering the inner product kernel and radial kernel introduced in \cref{sec:assum}, \citet{el2010spectrum} demonstrate that when $\vX\sim p$, the related kernel matrix $\vK(\vX,\vX)$ in high dimensions can be well approximated by ${\vK^{{\rm lin}}}(\vX,\vX)$ (detailed later) in spectral norm. We state the approximation in \cref{lemma:KXX_approx} as below. 
This result will help us to disentangle the nonlinearity of kernel functions in high dimensions.
\begin{lemma}[\citet{el2010spectrum}]\label{lemma:KXX_approx} Assuming the kernel $K$ is the inner-product kernel or the radial kernel, and the training data $\vX\sim p$, under \cref{assumption:h,assumption:8+m}, we have
\begin{align*}
    \| \vK(\vX,\vX) - {\vK^{{\rm lin}}}(\vX,\vX) \|_2 \rightarrow 0\,,
\end{align*}
as $n,d \to \infty,n/d\to \zeta$, where ${\bm K^{{\rm lin}}}(\vX,\vX)$ is defined by
\begin{equation}\label{eq:linear_K}
{\bm K^{{\rm lin}}}(\vX,\vX) := \alpha_p \mathbbm{1}\mathbbm{1}^\top +\beta_p \frac{\vX \vX^\top}{d} + \gamma_p \vI + \vT_p\,,
\end{equation}
with non-negative parameters $\alpha_p$, $\beta_p$, $\gamma_p$, and the additional matrix $\vT_p$ given in Table~\ref{tab:param}. 

\end{lemma}
We can see that, the kernel matrix in high dimensions can be mainly approximated by its covariance matrix with an implicit regularization term $\gamma_p \vI$.
Besides, the positive-definiteness of $\vK^{\rm lin}$ can be guaranteed under \cref{assumption:h}, $\alpha_p,\beta_p,\gamma_p> 0$. 
By \cref{assumption:h}, we can directly derive $\alpha_p,\beta_p>0$. 
    $\exists \delta,\delta'\in [0,1]$, for the inner-product kernels, such that $\gamma_p = h''(\delta \tau_p) \tau_p^2/2>0$; and for the radial kernels, such that $\gamma_p = 2h''(-2\delta'\tau_p) \tau_p^2>0$. 

In the presence of covariate shifts, where the training data $\vX$ is sampled from $p$ and the test data $\vx$ is sampled from $q$, the approximation of the related kernel vector $\vK(\vX,\vx)$ involves $q$, and thus previous expansion \cite{el2010spectrum} cannot be directly applied to our setting. 
In this case, we state the relation in \cref{lemma:KXx_approx}, which additionally relies on \cref{assumption:bounded_shift}, with the proof deferred to \cref{app:approx}. 

\begin{lemma}\label{lemma:KXx_approx} Under \cref{assumption:h,assumption:8+m,assumption:bounded_shift}, where $c_{pq}<2\theta_q-1/2$, with the training data $\vX\sim p$ and a test data $\vx \sim q$, we have
\begin{align*}
    \E_q \|\vK(\vX,\vx)-\vK^{\rm lin}(\vX,\vx)\|_2\to 0\,, 
\end{align*}
as $n,d \to \infty,n/d\to \zeta$, where ${\bm K^{{\rm lin}}}(\vx,\vX)$ is defined by
\begin{align*}
    {\bm K^{{\rm lin}}}(\vX,\vx) := \beta_{pq} \frac{\vX \vx}{d} + \vT_{pq}(\vX,\vx)\,,
\end{align*}
with non-negative parameters $\beta_{pq}$, and the additional vector $\vT_{pq}$ given in Table~\ref{tab:param}.
\end{lemma}
The approximation of $\vK(\vX,\vX)$ and $\vK(\vX,\vx)$ under covariate shift can help us estimate the variance and bias. To ensure the convergence of the residual term, we require $c_{pq}<2\theta_q-1/2$ in \cref{assumption:bounded_shift}. 

\begin{table}[t]
	\centering
\caption{Parameters of the linearized kernel ${\bm K^{{\rm lin}}}$ involved with the curvature of $h$, when $\vX\sim p$.}\label{tab:param}
\vspace{3mm}
\resizebox{\columnwidth}{!}{
	\begin{threeparttable}
		\begin{tabular}{cccccccccccccccccccc}
			\toprule
			Parameters & Inner-Product Kernels & Radial Kernels  \cr
			\midrule
			$\alpha_p$  &$h(0)+h^{\prime \prime}(0) \frac{\Tr\left(\bm \vSigma_{p}^{2}\right)}{2d^{2}}$ & $h(-2\tau_p)+2h^{\prime \prime}(-2\tau_p) \frac{\Tr\left(\bm \vSigma_{p}^{2}\right)}{d^{2}}$ \cr
			\midrule
			$\beta_p$  &$h^{\prime}(0)$ & $2h^{\prime}(-2\tau_p)$ \cr
			\midrule
			$\gamma_p$  &$h(\tau_p)-h(0)-\tau_p h^{\prime}(0)$ & $h(0)- 2\tau_p h^{\prime}(-2\tau_p) - h(-2\tau_p)$ \cr
			\midrule
			$\bm T_p$ & $\bm 0_{n \times n}$ & $-h^{\prime}(-2\tau_p) \vA + \frac{1}{2}h^{\prime \prime}(-2\tau_p) \vA \odot \vA$ \tnote{1} \cr
   			\midrule
			$\beta_{pq}$  &$h^{\prime}(0)$ & $2h^{\prime}(-(\tau_p+\tau_q))$ \cr
			\midrule
			$\bm T_{pq}$ & $\bm 0_{n \times 1}$ & $-h(-(\tau_p+\tau_q))\cdot\mathbbm{1} - \frac{\beta_{pq}}{2}\vA(\vX, \vx)$ \tnote{2} \cr
			\bottomrule
		\end{tabular}
		\begin{tablenotes}
			\item[1] $\vA := \mathbbm{1} \bm \psi^\top + \bm \psi \mathbbm{1}^\top$, where $\bm \psi \in \mathbb{R}^n$ with $\psi_i := \| \vx_i \|^2_2/d - \tau_p$.
            \item[2] $\vA(\vX, \vx) := \psi_{\vx} + \vpsi$, where $\psi_{\vx} = \| \vx \|^2_2/d - \tau_q$. 
		\end{tablenotes}
	\end{threeparttable}}
\end{table}

In the next, we are ready to present our results on the estimation for variance and bias.
For ease of analysis, we focus on the {\emph{inner product}} kernel for future estimation.
The results on radial kernels require additional efforts to control non-zero $\vT_{pq}$, which goes beyond the main target of this work.
\subsection{Variance estimation}
\label{sec:variance}

In this section, we present the estimation for the variance from the perspective of a data-dependent regularized kernel.
This helps us to have a better understanding of the role of re-weighting in variance.

\begin{theorem}[Variance: Data-dependent regularization]

\label{thm:variance}
Let $\delta\in (0,1)$, under \cref{assumption:h,assumption:8+m,assumption:bounded_shift}, then for large $d$, with probability at least $1 - \delta - 2d^{-2}$ with respect to a draw of $\vX\sim p$ and $\epsilon>0$, the variance can be estimated by
\begin{equation}\label{eq:variance}
\begin{aligned}
     {\sf V} \leq & \frac{8\sigma_\varepsilon^2 \| \vSigma_q \|}{d} \underbrace{\mathcal{N}\left(\frac{\vX\vX^\top}{d}+\frac{\lambda n}{\beta_p}{ \overline{\vW}({\vX})}^{-1}; \frac{\gamma_p}{\beta_p}\right)}_{\text{dominated term}~{\sf V}_{\vx}} \\
+& \frac{8\sigma_\varepsilon^2}{\gamma_p^2}d^{-(4\theta_q-1-2c_{pq})} \log^{4(1+\epsilon)} d \,. 
\end{aligned}
\end{equation}
\end{theorem}

{\bf Remark:} We do not need the boundedness (\cref{assumption:IW_assumption}) on $\overline{\vW}$ for the estimation of variance.
Nevertheless, \cref{assumption:bounded_shift}, with $c_{pq}< 2\theta_q-\frac{1}{2}$, is required to ensure the similarity between training and test distribution for the kernel approximation. Otherwise, a large difference between training and test distribution leads to the divergence of the residual term as $d\to\infty$. 
For example, when training and test data are sampled from two distributions with almost zero overlap, it will be impossible to generalize to the test distribution.
In the unshifted case ($\vSigma_p=\vSigma_q$ and $c_{pq} = 0$ in \cref{assumption:bounded_shift}) leads to the second term in Eq.~\eqref{eq:variance} admits a smaller value (or higher rate) at the order of $d^{-(4 \theta_q -1)}$. 

Since the first term ${\sf V}_{\vx}$ in Eq.~\eqref{eq:variance} dominates the estimation for variance, we detail this in the next part.

The dominated term in \cref{eq:variance} can be represented as
\begin{align}\label{eq:dominvariance}
    {\sf V}_{\vx} \asymp \frac{1}{d} \mathcal{N}\left(\frac{\vX\vX^\top}{d}+\frac{\lambda n}{\beta_p}{ \overline{\vW}({\vX})}^{-1}; \frac{\gamma_p}{\beta_p}\right)\,,
\end{align}
which implies that the variance is well controlled by the capacity of $\vK^{\rm lin} +\lambda n\overline{\vW}^{-1}$.
An intuitive example is to choose $(\overline{\vW})^{-1}$ by $c \vI$ with a large constant $c$ such that $\vK^{\rm lin} +n\lambda\overline{\vW}^{-1}$ has larger eigenvalues, allowing for a smaller effective dimension; and thus the variance (strictly speaking, its estimation) can decrease to some extent under this case.
In fact, since the re-weighting strategy $\overline{\vW}$ is quite general (not limited to the importance ratio $\vW$), there always exists suitable selection schemes that allow for a smaller $\mathcal{N}\left(\frac{\vX\vX^\top}{d}+\frac{\lambda n}{\beta_p}{ \overline{\vW}({\vX})}^{-1}; \frac{\gamma_p}{\beta_p}\right)$ (and smaller variance) in theory.

Besides, as a diagonal matrix $\overline{\bm W}$, each element $[\overline{\bm W}]_{ii}$ only affects the similarity of the data point ${\bm x}_i$ and itself. That means, the data points are ``importance reweighted'' but the similarity among different data points is unchanged.
In this case, importance weighting can be regarded as a special case of active learning and even data subsampling \cite{kolossov2024towards}. This motivates us to design more advanced active learning-based algorithms to select important data points, which is beneficial to handle covariate shifts in practice.

\subsection{Bias estimation}
\label{sec:bias}
In this subsection, we aim to derive the estimation for bias. We first present the spectral decomposition of the kernel to handle all scales of regularization parameter $\lambda>0$, see \cref{sec:bias-in}.
In the next, we analyze a special choice of regularization parameter $\lambda$, i.e., $\lambda\asymp n^{-c_\lambda}$, which stems from the classical analysis in the kernel literature~\citep{gogolashvili2023importance,ma2023optimally} and incorporates the dimension-dependent shifts to accommodate the high-dimensional setting, see \cref{sec:biasc}. 

\subsubsection{Bias under arbitrary regularization}
\label{sec:bias-in}
Here we present the bias estimation from the spectral decomposition of the kernel.
Note that the analysis in this part allows for any choice of regularization parameter. 
We consider the uniform boundedness of the re-weighting function and RKHS norm of the target function, i.e., a special case of \cref{assumption:sourcecon,assumption:IW_assumption}. Accordingly, we have the following theorem, with the proof deferred to \cref{app:biasproof}.

\begin{theorem}[Bias under arbitrary $\lambda$]
\label{lemma:bias}
Let $\delta\in(0,1)$, under \cref{assumption:h,assumption:8+m,assumption:bounded_shift}, \cref{assumption:sourcecon} with $\bar{r} = \frac{1}{2},c_{\mathcal{H}}=0$, \cref{assumption:IW_assumption} for bounded ratio: $ \overline{w}(\vx), w(\vx) \leq W_{\max}$. 
We have the bias ${\sf B}$ is upper bounded as
${\sf B}\leq {\sf B}_{\rm in}+ {\sf B}_{\rm iw}\,,$
where ${\sf B}_{\rm in}$ is the \emph{intrinsic bias} that only depends on the problem of covariate shift from $p$ to $q$ via the ratio ${w}(\vx)$ 
\begin{align*}
    {\sf B}_{\rm in} \! := \! \Tr\left(\vK^{\rm lin}{\vW} \right)/n \,.
\end{align*}
The second term is the re-weighting bias ${\sf B}_{\rm iw}$ that depends on the choice of $\overline{w}(\vx)$, $w(\vx)$, and $\lambda$, for $\epsilon>0$,  

\begin{equation*}
\begin{split}
& {\sf B}_{\rm iw} \!:=\! {4\lambda^2 n}\Tr\left(\left(\lambda n\mI \!+\! {{\vK^{\rm lin}\overline{\vW}}}\right)^{-2} {{\vK^{\rm lin}{\vW}}}\right) \!+\! \lambda^2 \kappa W_{\max} \\
& + 6\kappa W_{\max}\sqrt{\frac{\log 1/\delta}{2n}} + \widetilde{C} d^{-\theta_p} (\delta^{-1/2} + \log^{\frac{1+\epsilon}{2}} d) W_{\max},
    \end{split}
\end{equation*}
with probability at least $1 - 4\delta$ for sufficiently large $d$. 
\end{theorem}
{\bf Remark:} We make the following remarks:

\textit{1)} In our analysis, the first term ${\sf B}_{\rm in}$ describes the intrinsic bias of the distribution shift problem, in a constant order, which is independent of any specific re-weighting way. This coincides with results from high dimensional statistics for interpolation learning, e.g., \cite{hastie2022surprises,liang2020just}.

\textit{2)} The second term ${\sf B}_{\rm iw}$ involves the re-weighting strategy and its original ratio, which contributes to the importance re-weighting bias.
Since $\overline{\vW}$ can be chosen quite generally, it allows for a smaller $ {\sf B}_{\rm in}$ to some extent. More importantly, as $\lambda\to 0$, the re-weighting bias ${\sf B}_{\rm iw}$ will be close to zero.

Further, if we choose the re-weighting function $\overline{w}$ with the ratio $w$, then the estimation for the bias in \cref{lemma:bias} can be simplified as below,
\cref{app:biasproof}. 
\begin{corollary}[Bias: $\overline{w}=w$]
\label{bias:ratio}
    Under the same setting of \cref{lemma:bias}, choosing the re-weighting strategy $\overline{w}$ as the ratio $w$, and $\lambda=o(1)$, for sufficiently large $d$, the bias can be simplified as
    \begin{equation*}\label{eq:bratio}
     {\sf B} \lesssim  \frac{\Tr(\vK^{\rm lin} \vW)}{n} +  \lambda^2 n\mathcal{N}\left({\vK^{\rm lin} \vW},n\lambda\right)  + o(1)\,, w.h.p\,,
\end{equation*}
\end{corollary}

We can see that the bias term is controlled by the spectral decay of the re-weighting kernel matrix $\vK^{\rm lin}\vW$ via the importance ratio.

{\bf Discussion on excess risk:}
Combining Eq.~\eqref{eq:dominvariance} and \cref{lemma:bias}, taking $\lambda = o(1)$, the summation of the bias and variance (i.e., the excess risk) admits
\begin{equation*}
\begin{split}
        &{\sf B + V} \approx {\sf B}_{\rm in} + {\sf V}_{\vx} \\
        \lesssim& \frac{\Tr(\vK^{\rm lin} \vW)}{n} +  \frac{1}{d} \mathcal{N}\left(\frac{\vX\vX^\top}{d}+\frac{\lambda n}{\beta_p}{ \overline{\vW}({\vX})}^{-1}; \frac{\gamma_p}{\beta_p}\right) \,.
\end{split}
\end{equation*}
There exists a trade-off between the intrinsic bias ${\sf B}_{\rm in}$ and the dominated term ${\sf V}_{\vx}$ in variance: a suitable $\overline{\vW}$ that can be chosen generally, allows for a decreasing variance but the intrinsic bias ${\sf B}_{\rm in}$ will not decrease due to the covariate shift problem itself, determined by ${w}(\vx) = \rd q(\vx) / \rd p(\vx)$.
Nevertheless, at least, under re-weighting, the variance and re-weighting bias can be decreased; the estimator can still generalize well, and the convergence rate is unchanged.

\subsubsection{Bias under well-chosen regularization}
\label{sec:biasc}
We follow the classical analysis for kernel methods (which does not require the high dimension condition) to derive the estimation for bias. This analysis cannot deal with the situation where $\lambda\to 0$. 

We start by defining the data-free limit of \cref{eq:iw_emp_risk_X}. 
Denote the data-free limit of $\overline{f}_{\lambda,\vX}$ by $\overline{f}_{\lambda}$, 
\begin{align*}
\overline{f}_{\lambda}&:=\argmin_{f \in \mathcal{H}}\left\{\left\|f-f_\rho\right\|_{\overline{q}}^{2}+\lambda \|f\|_{\mathcal{H}}^{2}\right\}\,,
\end{align*}
then the solution $\overline{f}_\lambda$ in the data-free limit can be written as
\begin{align*}
\overline{f}_{\lambda}&=\left(L_{\overline{q}}+\lambda I \right)^{-1}L_{\overline{q}} f_{\rho}\,.
\end{align*}
Accordingly, the bias can be decomposed into 
\begin{align*}
    {\sf B} \leq \|\overline{f}_{\lambda,\vX} - \overline{f}_\lambda\|_q + \|\overline{f}_{\lambda} - f_\rho\|_q := {\sf B}_{\rm data} + {\sf B}_{\lambda}\,, 
\end{align*}
where ${\sf B}_{\rm data}$ denotes the data-dependent bias from $\vX\sim p$, and ${\sf B}_{\lambda}$ denotes the (data-free) regularization bias by $\lambda>0$. 

We present the estimation for the bias under a (not small) regularization parameter to balance ${\sf B}_{\rm data}$ and ${\sf B}_{\lambda}$, see the proof in \cref{app:biasproof}. The assumptions here are weaker than those in \cref{lemma:bias}, exemplified by the assumptions on the source condition (\cref{assumption:sourcecon}) and the upper bound of the density ratio (\cref{assumption:IW_assumption}). 
\begin{theorem}[Bias]\label{thm:bias}
Under \cref{assumption:h,assumption:8+m,assumption:sourcecon,assumption:effective_dimension,assumption:IW_assumption} with $\overline{r}\in[\frac{1}{2},1)$, $E_{q},E_{\overline{q}}>0$, $s_{q},s_{\overline{q}}\in [0,1]$, $t_{w},t_{\overline{w}}\in [0,1]$, $W_{w}(d),W_{\overline{w}}(d),\sigma_w(d),\sigma_{\overline{w}}(d)\geq 0$, $c_{w,1},c_{w,2},c_{\overline{w},1},c_{\overline{w},2}\geq 0$, and $c_{\mathcal{H}}\geq 0$.  
When $n,d\to\infty,n/d\to\zeta$, for any $\delta\in(0,1)$, let $\overline{A}= t_{\overline{w}}+(1-t_{\overline{w}})s_{\overline{q}}$ and the following two scalars $c_\lambda, C_{\lambda}$, 
    \begin{align*}
        c_\lambda  := \frac{1-4c_{\overline{w},2}}{2\overline{r}+ \overline{A}} \,,
    \end{align*}
    and
    \begin{align*}
    C_\lambda^{(1+\overline{A})s_{\overline{q}}} \geq 64(W_{\overline{w}}+\sigma_{\overline{w}}^2) E_{\overline{q}}^{2(1-t_{\overline{w}})} (2/\zeta)^{2c_{\overline{w},2}}\log^2(6/\delta)\,. 
    \end{align*}
    Choosing $\lambda :=\cdot C_\lambda n^{-c_\lambda}$, then with probability at least $1-\delta$, for sufficiently large $d$, when $c_{\mathcal{H}}<\overline{r} c_\lambda$, it holds that 
    \begin{align*}
        {\sf B}\lesssim n^{-\overline{r}c_\lambda+c_{\mathcal{H}}} \|L_q(L_{\overline{q}}+\lambda )^{-1}\|^{1/2}\,.  
    \end{align*}
For general $\lambda$, we have, with $\lesssim$ here hiding the dependence on $n$, 
\begin{align*}
    {\sf B}\lesssim (\lambda^{\overline{r}}+\lambda^{-\frac{1}{2}}) \|L_q(L_{\overline{q}}+\lambda)^{-1}\|^{\frac{1}{2}}\,. 
\end{align*} 
\end{theorem}

{\bf Remark:} As shown in the proof, when $\lambda\to 0$, the upper bound in \cref{thm:bias} will diverge ${O}(\lambda^{-1/2})$; and when $\lambda\to\infty$, the upper bound in \cref{lemma:bias} will diverge $O(\lambda^2)$.  Therefore, we can combine \cref{thm:bias,lemma:bias}: 
\begin{equation*}
     {\sf B}\lesssim  \min\{ (\lambda^{\overline{r}}+\lambda^{-\frac{1}{2}}) \|L_q(L_{\overline{q}}+\lambda)^{-1}\|^{\frac{1}{2}}, {\sf B}_{\rm in} + {\sf B}_{\rm iw} \}\,.
\end{equation*}
That means, under the assumption of \cref{lemma:bias}, if the regularization parameter decays to 0 with a certain power of $n$, then \cref{thm:bias} provides a good estimation, where the bias converges to zero. 
If $\lambda$ decays much faster and is sufficiently close to 0, we adopt \cref{lemma:bias}, which provide a uniform upper bound.

\section{Conclusion}
\label{sec:conclusion}

In this work, we provide a refined analysis on high dimensional kernel ridge regression under covariate shifts. Our results provide a non-asymptotic expansion of inner-product and radial kernels in high dimensions under covariate shifts. Our results on variance show that, the variance can be well controlled by the capacity of the data-dependent regularized kernel. Our results on bias give a thorough analysis, demonstrating that the intrinsic bias cannot be decreased but the re-weighting bias can tend to zero if the regularization term is sufficiently small. One limitation of this work is that our results only provide the upper bounds as well as empirical validation in \cref{sec:exp} but no exact formulation of the bias and variance.
This is because RMT cannot be directly applied to our setting when involving the IW strategy.
Nevertheless, our estimation still provides interesting findings to understand the role of re-weighting in terms of bias-variance trade-off.

\section*{Acknowledgements}
This work was carried out when YC was an intern in the EPFL LIONS group. This work was supported by Hasler Foundation Program: Hasler Responsible AI (project number 21043), the Army Research Office and was accomplished under Grant Number W911NF-24-1-0048, and Swiss National Science Foundation (SNSF) under grant number 200021\_205011. 
TS was partially supported by JSPS KAKENHI (24K02905) and JST CREST (JPMJCR2115, JPMJCR2015).
Corresponding author: Fanghui Liu.

\section*{Impact statement}
In this work, we study the role of re-weighting strategy in a high-capacity model, i.e., kernel ridge regression in high dimensions. 
Since this work is theoretical, there is no potential implications in security or trustworthy machine learning.

\bibliographystyle{icml2024}
\bibliography{references}

\begin{thebibliography}{46}
\providecommand{\natexlab}[1]{#1}
\providecommand{\url}[1]{\texttt{#1}}
\expandafter\ifx\csname urlstyle\endcsname\relax
  \providecommand{\doi}[1]{doi: #1}\else
  \providecommand{\doi}{doi: \begingroup \urlstyle{rm}\Url}\fi

\bibitem[Aerni et~al.(2023)Aerni, Milanta, Donhauser, and Yang]{aerni2023strong}
Aerni, M., Milanta, M., Donhauser, K., and Yang, F.
\newblock Strong inductive biases provably prevent harmless interpolation.
\newblock \emph{arXiv preprint arXiv:2301.07605}, 2023.

\bibitem[Boucheron et~al.(2013)Boucheron, Lugosi, and Massart]{boucheronconcentration}
Boucheron, S., Lugosi, G., and Massart, P.
\newblock Concentration inequalities: A nonasymptotic theory of independence,(2013), 2013.

\bibitem[Caponnetto \& De~Vito(2007)Caponnetto and De~Vito]{caponnetto2007optimal}
Caponnetto, A. and De~Vito, E.
\newblock Optimal rates for the regularized least-squares algorithm.
\newblock \emph{Foundations of Computational Mathematics}, 7\penalty0 (3):\penalty0 331--368, 2007.

\bibitem[Cortes et~al.(2010)Cortes, Mansour, and Mohri]{cortes2010}
Cortes, C., Mansour, Y., and Mohri, M.
\newblock Learning bounds for importance weighting.
\newblock In \emph{Advances in Neural Information Processing Systems}, pp.\  442--450, 2010.

\bibitem[Cucker \& Zhou(2007)Cucker and Zhou]{cucker2007learning}
Cucker, F. and Zhou, D.~X.
\newblock \emph{Learning theory: an approximation theory viewpoint}, volume~24.
\newblock Cambridge University Press, 2007.

\bibitem[Dicker(2016)]{dicker2016ridge}
Dicker, L.~H.
\newblock Ridge regression and asymptotic minimax estimation over spheres of growing dimension.
\newblock \emph{arXiv preprint arXiv:1601.03900}, 2016.

\bibitem[Dobriban \& Wager(2018)Dobriban and Wager]{dobriban2018high}
Dobriban, E. and Wager, S.
\newblock High-dimensional asymptotics of prediction: Ridge regression and classification.
\newblock \emph{The Annals of Statistics}, 46\penalty0 (1):\penalty0 247--279, 2018.

\bibitem[El~Karoui(2010)]{el2010spectrum}
El~Karoui, N.
\newblock The spectrum of kernel random matrices.
\newblock \emph{Ann. Statist.}, 38\penalty0 (1):\penalty0 1--50, 2010.

\bibitem[Fang et~al.(2020)Fang, Lu, Niu, and Sugiyama]{fang2020rethinking}
Fang, T., Lu, N., Niu, G., and Sugiyama, M.
\newblock Rethinking importance weighting for deep learning under distribution shift.
\newblock In \emph{Proceedings of the 34th International Conference on Neural Information Processing Systems}, pp.\  11996--12007, 2020.

\bibitem[Feng et~al.(2023)Feng, He, Wang, Wang, and Zhang]{feng2023towards}
Feng, X., He, X., Wang, C., Wang, C., and Zhang, J.
\newblock Towards a unified analysis of kernel-based methods under covariate shift.
\newblock \emph{arXiv preprint arXiv:2310.08237}, 2023.

\bibitem[Ge et~al.(2023)Ge, Tang, Fan, Ma, and Jin]{ge2023maximum}
Ge, J., Tang, S., Fan, J., Ma, C., and Jin, C.
\newblock Maximum likelihood estimation is all you need for well-specified covariate shift.
\newblock \emph{arXiv preprint arXiv:2311.15961}, 2023.

\bibitem[Ghorbani et~al.(2019)Ghorbani, Mei, Misiakiewicz, and Montanari]{ghorbani2021linearized}
Ghorbani, B., Mei, S., Misiakiewicz, T., and Montanari, A.
\newblock Linearized two-layers neural networks in high dimension.
\newblock \emph{arXiv preprint arXiv:1904.12191}, 2019.

\bibitem[Ghorbani et~al.(2020)Ghorbani, Mei, Misiakiewicz, and Montanari]{ghorbani2020neural}
Ghorbani, B., Mei, S., Misiakiewicz, T., and Montanari, A.
\newblock When do neural networks outperform kernel methods?
\newblock \emph{Advances in Neural Information Processing Systems}, 33:\penalty0 14820--14830, 2020.

\bibitem[Ghosh et~al.(2021)Ghosh, Mei, and Yu]{ghosh2021three}
Ghosh, N., Mei, S., and Yu, B.
\newblock The three stages of learning dynamics in high-dimensional kernel methods.
\newblock \emph{arXiv preprint arXiv:2111.07167}, 2021.

\bibitem[Gogolashvili et~al.(2023)Gogolashvili, Zecchin, Kanagawa, Kountouris, and Filippone]{gogolashvili2023importance}
Gogolashvili, D., Zecchin, M., Kanagawa, M., Kountouris, M., and Filippone, M.
\newblock When is importance weighting correction needed for covariate shift adaptation?
\newblock \emph{arXiv preprint arXiv:2303.04020}, 2023.

\bibitem[Hastie et~al.(2022)Hastie, Montanari, Rosset, and Tibshirani]{hastie2022surprises}
Hastie, T., Montanari, A., Rosset, S., and Tibshirani, R.~J.
\newblock Surprises in high-dimensional ridgeless least squares interpolation.
\newblock \emph{Annals of statistics}, 50\penalty0 (2):\penalty0 949, 2022.

\bibitem[Huang et~al.(2006)Huang, Smola, Gretton, Borgwardt, and Scholkopf]{huang2006correcting}
Huang, J., Smola, A.~J., Gretton, A., Borgwardt, K.~M., and Scholkopf, B.
\newblock Correcting sample selection bias by unlabeled data.
\newblock In \emph{Proceedings of the 19th International Conference on Neural Information Processing Systems}, pp.\  601--608, 2006.

\bibitem[Jacot et~al.(2018)Jacot, Gabriel, and Hongler]{jacot2018neural}
Jacot, A., Gabriel, F., and Hongler, C.
\newblock Neural tangent kernel: Convergence and generalization in neural networks.
\newblock \emph{Advances in Neural Information Processing Systems}, 31, 2018.

\bibitem[Jacot et~al.(2020)Jacot, Simsek, Spadaro, Hongler, and Gabriel]{jacot2020kernel}
Jacot, A., Simsek, B., Spadaro, F., Hongler, C., and Gabriel, F.
\newblock Kernel alignment risk estimator: Risk prediction from training data.
\newblock \emph{Advances in neural information processing systems}, 33:\penalty0 15568--15578, 2020.

\bibitem[Karoui(2013)]{karoui2013asymptotic}
Karoui, N.~E.
\newblock Asymptotic behavior of unregularized and ridge-regularized high-dimensional robust regression estimators: rigorous results.
\newblock \emph{arXiv preprint arXiv:1311.2445}, 2013.

\bibitem[Kolossov et~al.(2024)Kolossov, Montanari, and Tandon]{kolossov2024towards}
Kolossov, G., Montanari, A., and Tandon, P.
\newblock Towards a statistical theory of data selection under weak supervision.
\newblock In \emph{The Twelfth International Conference on Learning Representations}, 2024.

\bibitem[Kpotufe \& Martinet(2021)Kpotufe and Martinet]{kpotufe2021marginal}
Kpotufe, S. and Martinet, G.
\newblock Marginal singularity and the benefits of labels in covariate-shift.
\newblock \emph{The Annals of Statistics}, 49\penalty0 (6):\penalty0 3299--3323, 2021.

\bibitem[Liang \& Rakhlin(2020)Liang and Rakhlin]{liang2020just}
Liang, T. and Rakhlin, A.
\newblock Just interpolate: Kernel “ridgeless” regression can generalize.
\newblock \emph{THE ANNALS}, 48\penalty0 (3):\penalty0 1329--1347, 2020.

\bibitem[Liu et~al.(2021)Liu, Liao, and Suykens]{liu2021kernel}
Liu, F., Liao, Z., and Suykens, J.
\newblock Kernel regression in high dimensions: Refined analysis beyond double descent.
\newblock In \emph{International Conference on Artificial Intelligence and Statistics}, pp.\  649--657. PMLR, 2021.

\bibitem[Lu et~al.(2023)Lu, Zhang, Li, Xu, and Lin]{lu2023optimal}
Lu, W., Zhang, H., Li, Y., Xu, M., and Lin, Q.
\newblock Optimal rate of kernel regression in large dimensions.
\newblock \emph{arXiv preprint arXiv:2309.04268}, 2023.

\bibitem[Ma et~al.(2023)Ma, Pathak, and Wainwright]{ma2023optimally}
Ma, C., Pathak, R., and Wainwright, M.~J.
\newblock Optimally tackling covariate shift in rkhs-based nonparametric regression.
\newblock \emph{The Annals of Statistics}, 51\penalty0 (2):\penalty0 738--761, 2023.

\bibitem[McRae et~al.(2022)McRae, Karnik, Davenport, and Muthukumar]{mcrae2022harmless}
McRae, A.~D., Karnik, S., Davenport, M., and Muthukumar, V.~K.
\newblock Harmless interpolation in regression and classification with structured features.
\newblock In \emph{International Conference on Artificial Intelligence and Statistics}, pp.\  5853--5875. PMLR, 2022.

\bibitem[Mei et~al.(2021)Mei, Misiakiewicz, and Montanari]{mei2021learning}
Mei, S., Misiakiewicz, T., and Montanari, A.
\newblock Learning with invariances in random features and kernel models.
\newblock In \emph{Conference on Learning Theory}, pp.\  3351--3418. PMLR, 2021.

\bibitem[Mei et~al.(2022)Mei, Misiakiewicz, and Montanari]{mei2022generalization}
Mei, S., Misiakiewicz, T., and Montanari, A.
\newblock Generalization error of random feature and kernel methods: Hypercontractivity and kernel matrix concentration.
\newblock \emph{Applied and Computational Harmonic Analysis}, 59:\penalty0 3--84, 2022.

\bibitem[Mercer(1909)]{mercer1909xvi}
Mercer, J.
\newblock Xvi. functions of positive and negative type, and their connection the theory of integral equations.
\newblock \emph{Philosophical transactions of the royal society of London. Series A, containing papers of a mathematical or physical character}, 209\penalty0 (441-458):\penalty0 415--446, 1909.

\bibitem[Misiakiewicz \& Mei(2022)Misiakiewicz and Mei]{misiakiewicz2022learning}
Misiakiewicz, T. and Mei, S.
\newblock Learning with convolution and pooling operations in kernel methods.
\newblock \emph{Advances in Neural Information Processing Systems}, 35:\penalty0 29014--29025, 2022.

\bibitem[Nakkiran et~al.(2020)Nakkiran, Venkat, Kakade, and Ma]{nakkiran2020optimal}
Nakkiran, P., Venkat, P., Kakade, S., and Ma, T.
\newblock Optimal regularization can mitigate double descent.
\newblock \emph{arXiv preprint arXiv:2003.01897}, 2020.

\bibitem[Pathak et~al.(2022)Pathak, Ma, and Wainwright]{pathak2022new}
Pathak, R., Ma, C., and Wainwright, M.
\newblock A new similarity measure for covariate shift with applications to nonparametric regression.
\newblock In \emph{International Conference on Machine Learning}, pp.\  17517--17530. PMLR, 2022.

\bibitem[Shimodaira(2000)]{shimodaira2000}
Shimodaira, H.
\newblock Improving predictive inference under covariate shift by weighting the log-likelihood function.
\newblock \emph{Journal of Statistical Planning and Inference}, 90\penalty0 (2):\penalty0 227--244, 2000.

\bibitem[Smale \& Zhou(2004)Smale and Zhou]{smale2004shannon}
Smale, S. and Zhou, D.-X.
\newblock Shannon sampling and function reconstruction from point values.
\newblock \emph{Bulletin of the American Mathematical Society}, 41\penalty0 (3):\penalty0 279--305, 2004.

\bibitem[Smale \& Zhou(2007)Smale and Zhou]{smale2007learning}
Smale, S. and Zhou, D.-X.
\newblock Learning theory estimates via integral operators and their approximations.
\newblock \emph{Constructive approximation}, 26\penalty0 (2):\penalty0 153--172, 2007.

\bibitem[Sugiyama et~al.(2007)Sugiyama, Nakajima, Kashima, Buenau, and Kawanabe]{sugiyama2007direct}
Sugiyama, M., Nakajima, S., Kashima, H., Buenau, P., and Kawanabe, M.
\newblock Direct importance estimation with model selection and its application to covariate shift adaptation.
\newblock \emph{Advances in neural information processing systems}, 20, 2007.

\bibitem[Sugiyama et~al.(2008)Sugiyama, Suzuki, Nakajima, Kashima, Von~B{\"u}nau, and Kawanabe]{sugiyama2008direct}
Sugiyama, M., Suzuki, T., Nakajima, S., Kashima, H., Von~B{\"u}nau, P., and Kawanabe, M.
\newblock Direct importance estimation for covariate shift adaptation.
\newblock \emph{Annals of the Institute of Statistical Mathematics}, 60:\penalty0 699--746, 2008.

\bibitem[Sugiyama et~al.(2012)Sugiyama, Suzuki, and Kanamori]{sugiyama2012density}
Sugiyama, M., Suzuki, T., and Kanamori, T.
\newblock \emph{Density Ratio Estimation in Machine Learning}.
\newblock Cambridge University Press, 2012.

\bibitem[Tripuraneni et~al.(2021{\natexlab{a}})Tripuraneni, Adlam, and Pennington]{tripuraneni2021covariate}
Tripuraneni, N., Adlam, B., and Pennington, J.
\newblock Covariate shift in high-dimensional random feature regression.
\newblock \emph{arXiv preprint arXiv:2111.08234}, 2021{\natexlab{a}}.

\bibitem[Tripuraneni et~al.(2021{\natexlab{b}})Tripuraneni, Adlam, and Pennington]{tripuraneni2021overparameterization}
Tripuraneni, N., Adlam, B., and Pennington, J.
\newblock Overparameterization improves robustness to covariate shift in high dimensions.
\newblock \emph{Advances in Neural Information Processing Systems}, 34:\penalty0 13883--13897, 2021{\natexlab{b}}.

\bibitem[Vapnik(1999)]{vapnik1999nature}
Vapnik, V.
\newblock \emph{The nature of statistical learning theory}.
\newblock Springer science \& business media, 1999.

\bibitem[Wu \& Xu(2020)Wu and Xu]{wu2020optimal}
Wu, D. and Xu, J.
\newblock On the optimal weighted $\ell_2 $ regularization in overparameterized linear regression.
\newblock \emph{Advances in Neural Information Processing Systems}, 33:\penalty0 10112--10123, 2020.

\bibitem[Xiao et~al.(2022)Xiao, Hu, Misiakiewicz, Lu, and Pennington]{xiao2022precise}
Xiao, L., Hu, H., Misiakiewicz, T., Lu, Y., and Pennington, J.
\newblock Precise learning curves and higher-order scalings for dot-product kernel regression.
\newblock \emph{Advances in Neural Information Processing Systems}, 35:\penalty0 4558--4570, 2022.

\bibitem[Yu et~al.(2016)Yu, Suresh, Choromanski, Holtmann-Rice, and Kumar]{Yu2016Orthogonal}
Yu, F. X.~X., Suresh, A.~T., Choromanski, K.~M., Holtmann-Rice, D.~N., and Kumar, S.
\newblock Orthogonal random features.
\newblock \emph{Advances in neural information processing systems}, 29, 2016.

\bibitem[Zhai et~al.(2023)Zhai, Dan, Kolter, and Ravikumar]{zhai2023understanding}
Zhai, R., Dan, C., Kolter, J.~Z., and Ravikumar, P.~K.
\newblock Understanding why generalized reweighting does not improve over {ERM}.
\newblock In \emph{The Eleventh International Conference on Learning Representations}, 2023.

\end{thebibliography}
\clearpage
\onecolumn
\appendix

\section{Proofs}

\subsection{Bias-variance decomposition}
\label{app:biasvariance}
\begin{proof}[Proof of Lemma~\ref{lemma:decomposition}]
	Recall the closed-form solution of our IW estimator $\overline{f}_{\lambda,\vZ}(\vx)$ and its noiseless version $\overline{f}_{\lambda,\vX}(\vx)$, we have
	\begin{align*}
		\overline{f}_{\lambda,\vZ}(\vx) &=  \vK(\vx,\vX)(\vK(\vX,\vX)+\lambda n{ \overline{\vW}({\vX})}^{-1})^{-1}\vy\,.\\
  		\overline{f}_{\lambda,\vX}(\vx) &=  \vK(\vx,\vX)(\vK(\vX,\vX)+\lambda n{ \overline{\vW}({\vX})}^{-1})^{-1}f_\rho(\vX)\,. 
	\end{align*}
	Define $\veps := \vy - \E[\vy|\vX] = \vy - f_\rho(\vX)$, due to $\E_{y|\vX}(\veps) = 0$, we have 
	\begin{align*}
		\E_{\vy|\vX} (\overline{f}_{\lambda,\vZ}(\vx) - f_\rho)^2 &= \E_{\vy|\vX} \left(K(\vx,\vX) (\vK(\vX,\vX)+\lambda n\mI)^{-1} \veps \right)^2 + (\overline{f}_{\lambda,\vX}(\vx) - f_\rho(\vx))^2\,.
	\end{align*}
	Using Fubini's theorem, we have
	\begin{align*}
	 	&\E_{\vy|\vX} \| \overline{f}_{\lambda,\vZ} - f_\rho\|^2_q  =  \int \E_{\vy|\vX} (\overline{f}_{\lambda,\vZ}(\vx) - f_\rho)^2 \rd q(\vx)  =\E_{\vy|\vX} \|\overline{f}_{\lambda,\vZ}-\overline{f}_{\lambda,\vX}\|_{q}^2 + \|\overline{f}_{\lambda,\vX} - f_\rho\|_{q}^2\,,
	\end{align*}
which implies 
 \begin{align*}
     	 	&\E_{\vy|\vX} \| \overline{f}_{\lambda,\vZ} - f_\rho\|^2_q  =\E_{\vy|\vX} \|\overline{f}_{\lambda,\vZ}-\overline{f}_{\lambda,\vX}\|_{q}^2 + \|\overline{f}_{\lambda,\vX} - f_\rho\|_{q}^2\,. 
 \end{align*}
\end{proof}

\subsection{Approximation}
\label{app:approx}
\subsubsection{Inner-product kernel}
\label{app:approx-inner}
\begin{lemma}
\label{lemma:1byn}
Under \cref{assumption:8+m,assumption:bounded_shift}, and $\theta_p,\theta_q, c_{pq}$'s definitions, we have with probability at least $1 - d^{-2}$ with respect to the draw of $\vX\sim p$, for $\epsilon>0$ and $d$ large enough,
\begin{align*}
\E_q \| \vK(\vx, \vX) - \vK^{\rm lin} (\vx, \vX) \|^2
& \leq d^{-(4\theta_q-1-2c_{pq})} \log^{4(1+\epsilon)} d\,.
\end{align*}
\end{lemma}
\begin{proof}[Proof of \cref{lemma:1byn}]
The proof framework follows \citep[Lemma B.2]{liang2020just} but we need to provide a precise analysis to handle the covariance shift for $\vx \sim q$.
	Conditioned on $\vx_i, 1\leq i\leq n$, by Bernstein's inequality~\citep{boucheronconcentration}, with probability at least $1-\exp(-t)$ on $\vx\sim q$, for all $i \in [n]$, we have
	\begin{align*}
		\left| \frac{\bx^\top \vx_i}{d} \right| & = \left| \frac{\langle \vSigma_q^{1/2} \vx_i, \vSigma_q^{-1/2} \vx \rangle}{d} \right| \\
		&\leq \sqrt{\frac{2\| \vSigma_q^{1/2} \vx_i \|^2}{d}} \frac{\sqrt{t} + \log^{\frac{1+\epsilon}{2}} d}{\sqrt{d}} + \frac{1}{3} \frac{\| \vSigma_q^{1/2} \vx_i \|_\infty d^{\frac{2}{8+m_q}} (t+\log^{1+\epsilon} d)}{d}\\
		& \leq   \sqrt{\frac{2\|\vSigma_q^{1/2} \vx_i \|^2}{d}} \frac{\sqrt{t} + \log^{\frac{1+\epsilon}{2}} d}{\sqrt{d}} + \frac{1}{3} \frac{\| \vSigma_q^{1/2} \vx_i \| d^{\frac{2}{8+m_q}} (t+\log^{1+\epsilon} d)}{d}\\
		& = \frac{\sqrt{2} \| \vSigma_q^{1/2} \vx_i \|}{\sqrt{d}}\frac{\sqrt{t} + \log^{\frac{1+\epsilon}{2}} d}{\sqrt{d}} + \frac{1}{3} \frac{\| \vSigma_q^{1/2} \vx_i \|}{\sqrt{d}} d^{\frac{2}{8+m_q} - \frac{1}{2}} (t+\log^{1+\epsilon} d) \\
		&= \frac{\| \vSigma_q^{1/2} \vx_i \|}{\sqrt{d}} \left( \sqrt{2} d^{-1/2} (\sqrt{t} + \log^{\frac{1+\epsilon}{2}} d) + \frac{1}{3} d^{-\theta_q} (t+\log^{1+\epsilon} d) \right)\,,
	\end{align*}
where the first inequality uses \cref{assumption:8+m} such that 
\begin{align*}
    \max_k \left|[\vSigma_q^{1/2} \vx_i](k) \cdot [\vSigma_q^{-1/2} \vx](k) \right| \leq \| \vSigma_q^{1/2} \vx_i \|_\infty d^{\frac{2}{8+m_q}}.
\end{align*}
Applying \cref{lemma:quad-concentration} with \cref{assumption:8+m,assumption:bounded_shift}, 
we have
for all $j$, with probability at least $1-d^{-2}$ on $\mathcal{X}$
	\begin{align*}
		\max_i \frac{\|\vSigma_q^{1/2} \vx_i \|^2}{d} \leq \| \vSigma_q\| \max_i \frac{\|\vSigma_q^{-1/2} \vx_i \|^2}{d} =  \| \vSigma_q\| \max_i \frac{\|\vSigma_q^{-1/2}\vSigma_p^{1/2} \vSigma_p^{-1/2} \vx_i \|^2}{d} \lesssim \frac{\Tr(\vSigma_{pq})}{d} + d^{-\theta_p} \log^{\frac{1+\epsilon}{2}} d\,.
	\end{align*}

    We use the entry-wise Taylor expansion for the smooth kernel, let $\vx_i^\prime = c\vx + (1-c)\vx_i$ for some $c\in[0,1]$,  
\begin{align*}
    K(\vx, \vx_i) - K^{\rm lin} (\vx, \vx_i) &= \frac{h''(\vx_i^\prime)}{2}  \left(\frac{\vx^\top \vx_i}{d} \right)^2 \lesssim  \left(\frac{\vx^\top \vx_i}{d} \right)^2 \,.
\end{align*}
Therefore, with probability at least $1-\exp(-t)$ with respect to $\vx \sim q$, conditionally on $\vx_i\sim p, 1\leq i\leq n$, for sufficiently large $d$, 
\begin{equation}\label{eq:est_A1}
\begin{split}
    \|\vK(\vx, \vX) - \vK^{\rm lin} (\vx, \vX)\|
    &\lesssim \sqrt{d} \max_i    \left(\frac{\vx^\top \vx_i}{d} \right)^2\\
    &\lesssim \sqrt{d} \max_i \frac{\| \vSigma_q^{1/2} \vx_i \|^2}{d} \left( d^{-1} (t + \log^{1+\epsilon} d) + d^{-2\theta_q} (t^2+\log^{2(1+\epsilon)} d) \right)
    \\
    &\lesssim \sqrt{d} \max_i\frac{\| \vSigma_q^{1/2} \vx_i \|^2}{d} \left(d^{-2\theta_q} (t^2+\log^{2(1+\epsilon)} d) \right)\qquad\text{[since $\theta_q\leq\frac{1}{2}$]}\\
    &\lesssim d^{-2\theta_q+1/2} (t^2+\log^{2(1+\epsilon)} d)\left(\frac{\Tr(\vSigma_{pq})}{d} + d^{-\theta_p} \log^{\frac{1+\epsilon}{2}} d\right)\\
    &\lesssim d^{-2\theta_q+1/2+c_{pq}} (t^2+\log^{2(1+\epsilon)} d)
\end{split}
	\end{equation}
Define $z(t) := C \cdot  d^{-2\theta_q+1/2+c_{pq}} (t^2+\log^{2(1+\epsilon)} d)$, the above states that conditioned on $\vX$
\begin{align*}
    \mathbb{P}\left(\| \vK(\vx, \vX) - \vK^{\rm lin} (\vx, \vX) \|\geq z(t)\right) \leq 2\exp(-t),\quad \forall t>0\,. 
\end{align*}
Therefore, by the change of variables, we have
\begin{equation*}
\begin{split}
    		\E_{\vx \sim q} \| \vK(\vx, \vX) - \vK^{\rm lin} (\vx, \vX) \|^2 &=  \int_{\mathbb{R}_+} 2 z \cdot \mathbb{P}(\| \vK(\vx, \vX) - \vK^{\rm lin} (\vx, \vX) \| \geq z) \rd z \\
		&\leq C \int_{\mathbb{R}_+} d^{-4\theta_q+1+2c_{pq}} (t^2+\log^{2(1+\epsilon)} d)  \exp(-t) 2t \rd t   \\
		&\leq C \int_{\mathbb{R}_+} d^{-4\theta_q+1+2c_{pq}} t^3 \log^{2(1+\epsilon)} d   \exp(-t) \rd t \\
		& \lesssim d^{-(4\theta_q-1-2c_{pq})} \log^{4(1+\epsilon)} d\,, 
\end{split}
	\end{equation*}
	with probability at least $1-d^{-2}$ on $\vX$, for sufficiently large $d$. Here the constant is superseded by an additional  $\log^{2(1+\epsilon)} d$. Therefore, as long as \cref{assumption:bounded_shift} is satisfied, we have $4\theta_q-1-2c_{pq}>0$, and the residual term above will converge to 0 as $d\to \infty$. 
\end{proof}

\subsubsection{Radial kernel}
\begin{lemma}\label{lemma:prop-quad-concentration}
Let $\{\vx_i\}_{i=1}^n$ be i.i.d. random vectors in $\mathbb{R}^d$, whose entries are i.i.d., mean $0$, variance $1$ and $|x_i(k)| \leq C \cdot d^{\frac{2}{8+m}}$.  
	For any positive semi-definite matrices $\vSigma$
whose operator norms are uniformly bounded in $d$, and $n/d$ is asymptotically bounded, with $\theta = \frac{1}{2} - \frac{2}{8+m}$, with probability at least $1 - d^{-2}$, for $\epsilon>0$, we have
    	\begin{align*}
		\max_{i\neq j} \left| \frac{(\vx_i-\vx_j)^\top \vSigma (\vx_i-\vx_j)}{d} - 2 \frac{\Tr(\vSigma)}{d} \right| \leq 4 d^{-\theta}  \log^{\frac{1+\epsilon}{2}} d\,,
	\end{align*}
	for $d$ large enough.
\end{lemma}
\begin{proof}[Proof of \cref{lemma:prop-quad-concentration}]
    We write, for $i\neq j$, 
    \begin{align*}
        (\vx_i-\vx_j)^\top \vSigma (\vx_i-\vx_j) - 2\Tr(\vSigma) = \vx_i^\top\vSigma\vx_i + \vx_j^\top\vSigma\vx_j  -2\vx_i^\top\vSigma\vx_j\,.
    \end{align*}
    By \cref{lemma:quad-concentration}, for $i\neq j$, 
    \begin{align*}
        \left| \frac{\vx_i^\top \vSigma \vx_i}{d} -  \frac{\Tr(\vSigma)}{d} \right|\leq d^{-\theta}  \log^{\frac{1+\epsilon}{2}} d, \quad \left| \frac{\vx_i^\top \vSigma \vx_j}{d} \right|\leq d^{-\theta}  \log^{\frac{1+\epsilon}{2}} d\,. 
    \end{align*}
    Therefore,
\begin{align*}
\left| \frac{(\vx_i-\vx_j)^\top \vSigma (\vx_i-\vx_j)}{d} - 2 \frac{\Tr(\vSigma)}{d} \right| &= \left|\left(\frac{\vx_i^\top \vSigma\vx_i}{d} -  \frac{\Tr(\vSigma)}{d}\right)+\left(\frac{\vx_j^\top \vSigma\vx_j}{d} -  \frac{\Tr(\vSigma)}{d}\right) -2\frac{\vx_i^\top\vSigma\vx_j}{d}\right|\\
&\leq 4d^{-\theta}  \log^{\frac{1+\epsilon}{2}} d\,. 
\end{align*}
\end{proof}

\begin{lemma}
\label{lemma:2byn}
Under the \cref{assumption:h,assumption:8+m,assumption:bounded_shift}, and $\theta_p,\theta_q, c_{pq}$'s definitions, we have with probability at least $1 - 3d^{-2}$ with respect to the draw of $\vX\sim p$, for $d$ large enough,
\begin{align*}
\E_q \| \vK(\vx, \vX) - \vK^{\rm lin} (\vx, \vX) \|^2
&  \lesssim d^{-(4\min\{\theta_p,\theta_q-c_{pq}/2\}-1)}\log^{2(1+\epsilon)} d\,.
\end{align*}
\end{lemma}
\begin{proof}[Proof of \cref{lemma:2byn}]
	We start with the entry-wise Taylor expansion for the smooth kernel at $-(\tau_p+\tau_q)$
\begin{align*}
&K(\vx, \vx_j) = h\left(-\frac{1}{d}\| \vx - \vx_j \|_2^2\right) \\
=& h(-(\tau_p+\tau_q)) - h'(-(\tau_p+\tau_q)) \left(\frac{1}{d}\| \vx - \vx_j \|^2 - (\tau_p+\tau_q)\right) + \frac{h^{\prime \prime}(-(\tau_p+\tau_q))}{2} \left(\frac{1}{d}\| \vx - \vx_j \|^2 - (\tau_p+\tau_q) \right)^2 \\
+ & O(d^{-3/2}) \\
=& h(-(\tau_p+\tau_q)) - h'(-(\tau_p+\tau_q))\left(\psi_{\vx} + \psi_j - \frac{2\vx^{\!\top} \vx_j}{d}\right) + \frac{h^{\prime \prime}(\tau_p+\tau_q)}{2} \left(\psi_{\vx} + \psi_j - \frac{2\vx^{\!\top} \vx_j}{d} \right)^2  + O(d^{-3/2})\\
=&K^{\rm lin}(\vx, \vx_j) + \frac{h^{\prime \prime}(-(\tau_p+\tau_q))}{2} \left(\frac{1}{d}\| \vx - \vx_j \|_2^2 - (\tau_p+\tau_q) \right)^2 +O(d^{-3/2})\,,
\end{align*}
where $\psi_{j} = \| \vx_j \|^2_2/d - \tau_p$ for $j \in [n]$ as defined before.
We expand $\frac{1}{d}\| \vx - \vx_j \|_2^2 - (\tau_p+\tau_q) $ by
\begin{align*}
    \frac{1}{d}\| \vx - \vx_j \|_2^2 - (\tau_p+\tau_q)  =& \frac{\vx^\top\vx + \vx_i^\top\vx_i -2\vx^\top\vx_i-\Tr(\vSigma_p)-\Tr(\vSigma_q)}{d}\,,
\end{align*}
By a similar proof of \cref{eq:est_A1} in \cref{lemma:1byn}, conditioned on $\vx_i,1\leq i\leq n$, 
with probability at least $1-\exp(-t)$, 
\begin{align*}
   \left| \frac{\vx^\top\vx_i}{d}\right| \lesssim d^{-(\theta_q-c_{pq}/2)} (t+\log^{1+\epsilon} d)\,.
\end{align*}
Therefore, setting $t:=2\log d$, with probability at least $1-d^{-2}$, we have
\begin{align*}
   \left| \frac{\vx^\top\vx_i}{d}\right| \lesssim d^{-(\theta_q-c_{pq}/2)} (2\log d+\log^{1+\epsilon} d)\,.
\end{align*}
By \cref{lemma:quad-concentration,assumption:8+m}, and $\vx\sim q$, we have with probability at least $1-d^{-2}$, 
\begin{align*}
     \left| \frac{\vx^\top\vx}{d} - \frac{\Tr(\vSigma_q)}{d}\right|  = \left| \frac{(\vSigma_q^{-1/2}\vx)^\top\vSigma_q (\vSigma_q^{-1/2}\vx)}{d} - \frac{\Tr(\vSigma_q)}{d}\right|\leq d^{-\theta_q} \log^{\frac{1+\epsilon}{2}} d\,.
\end{align*}
By \cref{lemma:quad-concentration,assumption:8+m}, and $\vx_i\sim p$, we have with probability at least $1-d^{-2}$, 
\begin{align*}
    \left| \frac{\vx_i^\top\vx_i}{d} - \frac{\Tr(\vSigma_p)}{d}\right|= \left| \frac{(\vSigma_p^{-1/2}\vx_i)^\top\vSigma_p (\vSigma_p^{-1/2}\vx_i)}{d} - \frac{\Tr(\vSigma_p)}{d}\right| \leq d^{-\theta_p} \log^{\frac{1+\epsilon}{2}} d\,. 
\end{align*}
In total, with probability at least $1-3d^{-2}$, for sufficient large $d$, we have
\begin{align*}
   K(\vx,\vx_i) - K^{\rm lin}(\vx,\vx_i) \lesssim    \left( \frac{1}{d}\| \vx - \vx_j \|_2^2 - (\tau_p+\tau_q)\right)^2 \lesssim d^{-2\min\{\theta_p,\theta_q-c_{pq}/2\}} \log^{1+\epsilon} d\,,
\end{align*}
which leads to
\begin{align*}
    \E_q \| \vK(\vx, \vX) - \vK^{\rm lin} (\vx, \vX) \|^2 \lesssim d^{-(4\min\{\theta_p,\theta_q-c_{pq}/2\}-1)}\log^{2(1+\epsilon)} d\leq d^{-(4\min\{\theta_p,\theta_q-c_{pq}/2\}-1)} \log^{4(1+\epsilon)} d\,.
\end{align*}

By the definition of $\theta_p$ in \cref{def:distribution}, we have $\theta_p>\frac{1}{4}$. Therefore, as long as \cref{assumption:bounded_shift} is satisfied, we have $\theta_q-c_{pq}/2>\frac{1}{4}$, and the residual term above will converge to 0 as $d\to \infty$. 
\end{proof}

\subsection{Variance}
\subsubsection{Proof for variance}
\label{app:varproof}

\begin{lemma}[{\citet[Proposition A.1]{liang2020just}}]
	\label{lemma:quad-concentration}
	Let $\{\vx_i\}_{i=1}^n$ be i.i.d. random vectors in $\mathbb{R}^d$, whose entries are i.i.d., mean $0$, variance $1$ and $|x_i(k)| \leq C \cdot d^{\frac{2}{8+m}}$.  
	For any positive semi-definite matrices $\vSigma$
whose operator norms are uniformly bounded in $d$, and $n/d$ is asymptotically bounded, with $\theta = \frac{1}{2} - \frac{2}{8+m}$, we have with probability at least $1 - d^{-2}$, for $\epsilon>0$,
	\begin{align*}
		\max_{i,j} \left| \frac{\vx_i^\top \vSigma \vx_j}{d} - \delta_{ij} \frac{\Tr(\vSigma)}{d} \right| \leq d^{-\theta}  \log^{\frac{1+\epsilon}{2}} d\,,
	\end{align*}
	for $d$ large enough.
\end{lemma}

\begin{proof}[Proof of \cref{thm:variance} (Inner product kernels)]
	According to the definition of ${\sf V}$ and $\E[\vy|\vX] = f_\rho(\vX)$, we have
	\begin{align*}
		{\sf V} &=  \int \E_{\vy|\vX} \Tr\left(  K(\vx,\vX) (\vK(\vX,\vX)+\lambda n{ \overline{\vW}({\vX})}^{-1})^{-1} (\vy - f_\rho(\vX)) (\vy -  f_\rho(\vX))^\top\right.\\
  &\ \left.(\vK(\vX,\vX)+\lambda n{ \overline{\vW}({\vX})}^{-1})^{-1}  K(\vX,\vx) \right) \rd q(\vx) \\
		& \leq \int \|(\vK(\vX,\vX)+\lambda n{ \overline{\vW}({\vX})}^{-1})^{-1} K(\vX,\vx) \|^2 \| \E_{\vy|\vX} \left[(\vy - f_\rho(\vX)) (\vy -  f_\rho(\vX))^\top \right]\| \rd q(\vx)\,.
	\end{align*}
	Note that $\E_{\vy|\vX}\left[ (y_i - f_\rho(\vx_i))(y_j -f_\rho(\vx_j)) \right] = 0$ for $i \neq j$, and $\E_{\vy|\vX}\left[ (y_i - f_\rho(\vx_i))^2 \right] \leq \sigma_\varepsilon^2$, we have $\| \E_{\vy|\vX} \left[(\vy - f_\rho(\vX)) (\vy -  f_\rho(\vX))^\top \right]\| \leq \sigma_\varepsilon^2$. Accordingly, the variance under our IW estimator can be estimated by
	\begin{align*}
		\bV &\leq \sigma_\varepsilon^2 \int  \|(\vK(\vX,\vX)+\lambda n{ \overline{\vW}({\vX})}^{-1})^{-1} K(\vX,\vx) \|^2 \rd q(\vx) = \sigma_\varepsilon^2 \E_q \| (\vK(\vX, \vX)+\lambda n\overline{\vW}(\vX))^{-1} K(\vX, \vx) \|^2\,.
	\end{align*}
By \cref{tab:param} the following linearization of the inner product kernel holds: 
	\begin{align*}
		\vK^{\rm lin}(\vX, \vX) &:=  \gamma_p \mI + \alpha_p \mathbbm{1}\mathbbm{1}^\top + \beta_p \frac{\vX \vX^\top}{d} \in \mathbb{R}^{n \times n}, \\
		\vK^{\rm lin}(\vX, \vx) &:=  \beta_p \frac{\vX \vx}{d} \in \mathbb{R}^{n \times 1}, 
	\end{align*}
By \cref{assumption:8+m}, according to \citep[Proposition A.2]{liang2020just}, the kernel matrix admits the following asymmetric approximation with $\theta_p := \frac{1}{2} - \frac{2}{8+m_p}$,
\begin{align}\label{eq:KXX_lin_1}
 		\left\| \vK(\vX, \vX) - \vK^{\rm lin}(\vX, \vX) \right\| &\leq d^{-\theta_p} (\delta^{-1/2} + \log^{\frac{1+\epsilon}{2}} d)\,, \quad \text{w.p.}~ 1-\delta-d^{-2} \,.
\end{align}
The approximation $\left\| \vK(\vX, \vX) - \vK^{\rm lin}(\vX, \vX) \right\|$ is different from \citep[Lemma B.2]{liang2020just}, since training dataset $\vX$ is sampled from $p$ and the expectation under $q$. We prove this approximation under distribution shift in \cref{lemma:1byn}, such that 
\begin{align}\label{eq:KXx_lin_1}
    		\E_q \left\| \vK(\vx, \vX) -  \vK^{\rm lin}(\vx, \vX) \right\|^2 &\leq d^{-(4\theta_q-1-2c_{pq})} \log^{4(1+\epsilon)} d\,, \quad \text{w.p.}~ 1-d^{-2}\,.
\end{align}
By \cref{eq:KXX_lin_1}, as a direct consequence, one can see that for sufficiently large $d$, such that $d^{-\theta_p} (\delta^{-1/2} + \log^{\frac{1+\epsilon}{2}} d) \leq \gamma/2$, with probability $1-\delta - d^{-2}$, we have
\begin{align}
&\left\| (\vK+\lambda n\overline{\vW}^{-1})^{-1} \right\| \leq {\|\vK\|^{-1}\leq \frac{1}{\|\vK^{\rm lin}\|-d^{-\theta_p} (\delta^{-1/2} + \log^{\frac{1+\epsilon}{2}} d)}}\leq   \frac{1}{\gamma_p - d^{-\theta_p} (\delta^{-1/2} + \log^{\frac{1+\epsilon}{2}} d)}\leq \frac{2}{\gamma_p}, \label{eq:KXX_lin_1_1}\\
&\left\| (\vK+\lambda n\overline{\vW}^{-1})^{-1} (\vK^{\rm lin} + \lambda n\overline{\vW}^{-1}) \right\|\leq {\left\| (\vK+\lambda n\overline{\vW}^{-1})^{-1} (\vK + \lambda n\overline{\vW}^{-1} + \vK^{\rm lin}-\vK)  \right\|} \nonumber \\
&\leq 1 + \| (\vK+\lambda n\overline{\vW}^{-1})^{-1} \| \cdot \| \vK(\vX, \vX) -  \vK^{\rm lin}(\vX, \vX) \| \nonumber\\
&\leq 1 +  \frac{d^{-\theta_p} (\delta^{-1/2} + \log^{\frac{1+\epsilon}{2}} d)}{\gamma_p - d^{-\theta_p} (\delta^{-1/2} + \log^{\frac{1+\epsilon}{2}} d)}
\leq  \frac{\gamma_p}{\gamma_p - d^{-\theta_p} (\delta^{-1/2} + \log^{\frac{1+\epsilon}{2}} d)} \leq 2\,. \label{eq:KXX_lin_1_2}
\end{align}
Combining \cref{eq:KXx_lin_1,eq:KXX_lin_1_1,eq:KXX_lin_1_2}, the variance can be estimated by
\begin{equation}\label{eq:proof_V_1}
\begin{split}
        {\sf V} & \leq  \sigma_\varepsilon^2 \E_q \| (\vK(\vX, \vX)+\lambda n\overline{\vW}(\vX)^{-1})^{-1} \vK(\vX, \vx) \|^2\\
    & \leq 2\sigma_\varepsilon^2 \E_q \| (\vK(\vX, \vX)+\lambda n\overline{\vW}(\vX)^{-1})^{-1} \vK^{\rm lin}(\vX, \vx) \|^2 \\
     & + 2 \sigma_\varepsilon^2 \left\| (\vK(\vX, \vX)+\lambda n\overline{\vW}(\vX)^{-1})^{-1} \right\|^2 \cdot \E_q \|\vK(\vX, \vx) -  \vK^{\rm lin}(\vX, \vx)\|^2 \\
    & \leq 2\sigma_\varepsilon^2 \left\| (\vK(\vX, \vX)+\lambda{\overline{\vW}({\vX})}^{-1}
)^{-1} (\vK^{\rm lin}(\vX, \vX)+\lambda n\overline{\vW}(\vX)^{-1}) \right\|^2 \\
&\ \E_q \| (\vK^{\rm lin}(\vX, \vX)+\lambda n\overline{\vW}(\vX)^{-1})^{-1} \vK^{\rm lin}(\vX, \vx) \|^2  + \frac{8\sigma_\varepsilon^2}{\gamma_p^2} d^{-(4\theta_q-1-2c_{pq})} \log^{4(1+\epsilon)} d \\
    & \leq 8\sigma_\varepsilon^2  \E_q \| (\vK^{\rm lin}(\vX, \vX)+\lambda n\overline{\vW}(\vX)^{-1})^{-1} \vK^{\rm lin}(\vX, \vx) \|^2 +  \frac{8\sigma_\varepsilon^2}{\gamma_p^2} d^{-(4\theta_q-1-2c_{pq})} \log^{4(1+\epsilon)} d\,. 
\end{split}
\end{equation}
Besides, the IW estimator under the linearized kernel matrix leads to 
\begin{align*}
	 &\E_q \|  (\vK^{\rm lin}(\vX, \vX)+\lambda n{ \overline{\vW}({\vX})}^{-1})^{-1} \vK^{\rm lin}(\vX, \vx) \|^2\\
  &= 	 \E_{q} \Tr\left( \left[\gamma_p \mI+\lambda n{ \overline{\vW}({\vX})}^{-1}+ \alpha_p \mathbbm{1}\mathbbm{1}^\top + \beta_p \frac{\vX\vX^\top}{d} \right]^{-2}  \beta_p \frac{\vX \vx}{d}  \beta_p \frac{\vx^\top \vX^\top}{d}  \right) \\
	 	&= \Tr\left( \left[\gamma_p \mI+\lambda n{ \overline{\vW}({\vX})}^{-1} + \alpha_p \mathbbm{1}\mathbbm{1}^\top + \beta_p \frac{\vX\vX^\top}{d} \right]^{-2} \beta_p^2 \frac{\vX \vSigma_q \vX^\top}{d^2}  \right)\\
&\leq \frac{\|\vSigma_q\|}{d} \Tr\left( \left[\frac{\gamma_p}{\beta_p} \mI+\frac{\lambda n}{\beta_p}{ \overline{\vW}({\vX})}^{-1} +  \frac{\vX\vX^\top}{d} \right]^{-2}  \frac{\vX\vX^\top}{d} \right)\\
&= \frac{\|\vSigma_q\|}{d} \mathcal{N}\left(\frac{\vX\vX^\top}{d}+\frac{\lambda n}{\beta_p}{ \overline{\vW}({\vX})}^{-1}; \frac{\gamma_p}{\beta_p}\right)\,,
\end{align*}
with the following constants, $\beta_p = h'(0)=O(1)$, $\gamma_p=O((\tau_p)^2)$. 

Finally, combining previous results, with probability at least $1-\delta-2d^{-2}$, for sufficiently large $d$, we have
	\begin{align*}
		{\sf V}  \leq \frac{8\sigma_\varepsilon^2 \| \Sigma_q \|}{d} \mathcal{N}\left(\frac{\vX\vX^\top}{d}+\frac{\lambda n}{\beta_p}{ \overline{\vW}({\vX})}^{-1}; \frac{\gamma_p}{\beta_p}\right)+ \frac{8\sigma_\varepsilon^2}{\gamma_p^2} d^{-(4\theta_q-1-2c_{pq})} \log^{4(1+\epsilon)} d\,. 
	\end{align*}
\end{proof}

\subsection{Bias}
\label{app:biasproof}
In the next, we present the proof for the bias based on whether the used regularization parameter is small.
We firstly give the proof for \cref{thm:bias} and then \cref{lemma:bias}.

\begin{lemma}
	\label{lem:symmetrization}
	Let $g(\vx) \in \mathbb{R}$ that satisfies $\forall g \in \cG$, $|g(\vx)| \leq \kappa$ for all $\vx$. Then with probability at least $1-2\delta$, we have for i.i.d. $\vx_i \sim  q$
	\begin{align*}
		\sup_{g \in \cG} \left| \E g(\vx) - \widehat{\E}_n g(\vx) \right| &\leq \E \sup_{g \in \cG} \left| \E g(\vx) - \widehat{\E}_n g(\vx) \right| + \kappa
		 \sqrt{\frac{\log 1/\delta}{2n}} \\
		 & \leq 2\E \sup_{g \in \cG} \frac{1}{n}\sum_{i=1}^n \epsilon_i g(\vx_i) + \kappa
		 \sqrt{\frac{\log 1/\delta}{2n}} \\
		 & \leq 2\E_{\epsilon} \sup_{g \in \cG} \frac{1}{n}\sum_{i=1}^n \epsilon_i g(\vx_i) + 3 \kappa
		 \sqrt{\frac{\log 1/\delta}{2n}} \,,
	\end{align*}
	where $\E_{\epsilon}$ denotes the conditional expectation with respect to i.i.d. Rademacher random variables $\epsilon_1,\ldots,\epsilon_n$.
\end{lemma}

\subsubsection{Proof of \texorpdfstring{\cref{lemma:bias}}{}}

\begin{proof}[Proof of \cref{lemma:bias}]
For the bias, we use the spectral decomposition of the kernel. To be specific, denote $f_\rho(\vx) = \sum_{i=1} \phi_i(\vx) f_i $ with $f_i$ being the coefficients of $f$ under the basis $\phi_i(\vx)$, we can write it as $f(\vx) = \vphi(\vx)^\top \vf$ where $\vf = [f_1, f_2, \ldots, f_p]^\top$ can be a possibly infinite vector. 
Accordingly, the bias term can be formulated as
\begin{align*}
	\bB &= \int \left| \vphi^\top(\vx) \vLambda^{1/2} \left[  \vLambda^{1/2} \vphi(\vX) [\vphi(\vX)^\top \vLambda \vphi(\vX)+\lambda n  \overline{\vW}^{-1}]^{-1} \vphi(\vX)^\top \vLambda^{1/2} - \vI \right] \vLambda^{-1/2} \vf_\rho \right|^2 \rd q(\vx) \\
	&\leq \int \left\|  \left[  \vLambda^{1/2} \vphi(\vX) [\vphi(\vX)^\top \vLambda \vphi(\vX)+\lambda n  \overline{\vW}^{-1}]^{-1} \vphi(\vX)^\top \vLambda^{1/2} - \vI \right]  \vLambda^{1/2} \vphi(\vx) \right\|^2 \rd q(\vx) \cdot  \| \vLambda^{-1/2} \vf_\rho \|^2 \\
	&= \| f_\rho \|_{\cH}^2  \int \left\|  \left[  \vLambda^{1/2} \vphi(\vX) [\vphi(\vX)^\top \vLambda \vphi(\vX)+\lambda n \overline{\vW}^{-1}]^{-1} \vphi(\vX)^\top \vLambda^{1/2} - \vI \right]  \vLambda^{1/2} \vphi(\vx) \right\|^2 \rd q(\vx)\,.
\end{align*}
We note the following fact 
\begin{align*}
    &\left(\vLambda^{1/2} \vphi(\vX) [\vphi(\vX)^\top \vLambda \vphi(\vX)+\lambda n  \overline{\vW}({\vX})^{-1}]^{-1} \vphi(\vX)^\top \vLambda^{1/2} - \vLambda^{1/2} \vphi(\vX) [\vphi(\vX)^\top \vLambda \vphi(\vX)]^{-1} \vphi(\vX)^\top \vLambda^{1/2} \right)\\
    &\cdot \left(I- \vLambda^{1/2} \vphi(\vX) [\vphi(\vX)^\top \vLambda \vphi(\vX)]^{-1} \vphi(\vX)^\top \vLambda^{1/2} \right)\\
    = &\left(\vLambda^{1/2} \vphi(\vX) [\vphi(\vX)^\top \vLambda \vphi(\vX)+\lambda n \overline{\vW}({\vX})^{-1}]^{-1} \vphi(\vX)^\top \vLambda^{1/2} - \vLambda^{1/2} \vphi(\vX) [\vphi(\vX)^\top \vLambda \vphi(\vX)]^{-1} \vphi(\vX)^\top \vLambda^{1/2} \right)\\
    -& \vLambda^{1/2} \vphi(\vX) [\vphi(\vX)^\top \vLambda \vphi(\vX)+\lambda n \overline{\vW}({\vX})^{-1}]^{-1} \vphi(\vX)^\top \vLambda^{1/2} \vLambda^{1/2} \vphi(\vX) [\vphi(\vX)^\top \vLambda \vphi(\vX)]^{-1} \vphi(\vX)^\top \vLambda^{1/2}
    \\
    + & \vLambda^{1/2} \vphi(\vX) [\vphi(\vX)^\top \vLambda \vphi(\vX)]^{-1} \vphi(\vX)^\top \vLambda^{1/2} \vLambda^{1/2} \vphi(\vX) [\vphi(\vX)^\top \vLambda \vphi(\vX)]^{-1} \vphi(\vX)^\top \vLambda^{1/2}\\
    = &\ \mathbf{0}\,,
\end{align*}
with $A^{-1}-B^{-1}=B^{-1}(B-A)A^{-1}$, the main part in the bias term can be split into the following two terms
\begin{align*}
    & \int \left\|  \left[  \vLambda^{1/2} \vphi(\vX) [\vphi(\vX)^\top \vLambda \vphi(\vX)+\lambda n \overline{\vW}({\vX})^{-1}]^{-1} \vphi(\vX)^\top \vLambda^{1/2} - I \right]  \vLambda^{1/2} \vphi(\vx) \right\|^2 \rd q(\vx)\\
     = &\underbrace{\int \left\|  \left[  \vLambda^{1/2} \vphi(\vX) [\vphi(\vX)^\top \vLambda \vphi(\vX)]^{-1} \vphi(\vX)^\top \vLambda^{1/2} - I \right]  \vLambda^{1/2} \vphi(\vx) \right\|^2 \rd q(\vx)}_{\tt (A)}\\
     +&\underbrace{\int \left\|    \left[\vLambda^{1/2} \vphi(\vX) [\vphi(\vX)^\top \vLambda \vphi(\vX)]^{-1} \left[I + \vphi(\vX)^\top\vLambda \vphi(\vX)  \overline{\vW}({\vX})/(\lambda n) \right]^{-1} \vphi(\vX)^\top \vLambda^{1/2}  \right]  \vLambda^{1/2} \vphi(\vx) \right\|^2 \rd q(\vx)}_{\tt (B)} \,.
\end{align*}
We assume the SVD decomposition of $\vLambda^{\frac{1}{2}}\vphi(\vX)=\widehat{\vU}\widehat{\vSigma} \widehat{\vV}^\top, \widehat{\vU}\in\mathbb{R}^{p\times n}, \widehat{\vSigma}\in\mathbb{R}^{n\times n}, \widehat{\vV}\in\mathbb{R}^{n\times n}$ and the $\vK(\vX,\vX)=\vphi^\top(\vX) \vLambda \vphi(\vX)$ has full rank as mentioned in the main text. 

{\bf Part {\tt (A)}} is essentially ridgeless regression under the distribution shift. We modify the proof from \citet{liang2020just} by introducing the additional re-weighting quantity $\overline{w}(\vx)$. 

Denote the top $k$ columns of $\widehat{\vU}$ to be $\widehat{\vU}_{k}$, and $P_{\widehat{\vU}_k}^\perp$ to be projection to the eigenspace orthogonal to $\widehat{\vU}_{k}$. By observing that $\vLambda^{1/2} \vphi(\vX) (\vphi(\vX)^\top \vLambda \vphi(\vX))^{-1} \vphi(\vX)^\top \vLambda^{1/2}$ is a projection matrix, it is clear that for all $k\leq n$,
\begin{align}
	{\tt (A)} &\leq \| f_\rho \|_{\cH}^2  \int \left\| P^\perp_{\widehat{\vU}} \left(\vLambda^{1/2} \vphi(\vx) \right) \right\|^2 \rd q(\vx) \leq  \| f_\rho \|_{\cH}^2  \int \left\| P^\perp_{\widehat{\vU}_k} \left(\vLambda^{1/2} \vphi(\vx) \right) \right\|^2 \rd q(\vx)\,.
\end{align}
Denote the function $g$ indexed by any rank-$k$ projection $\vU_k$ as
\begin{align}
	g_{\vU_k}(\vx) :=  \left\| P_{\vU_k} \left(\vLambda^{1/2} \vphi(\vx) \sqrt{w(\vx)} \right) \right\|^2 = \Tr\left(w(\vx) \vphi^\top(\vx) \vLambda^{1/2} \vU_k \vU_k^\top \vLambda^{1/2} \vphi(\vx) \right)\,.
\end{align}
Clearly, $\| \vU_k \vU_k^\top \|_F = \sqrt{k}$.
Define the function class
\begin{align*}
	\cG_k := \{ g_{\vU_k}(\vx): \vU_k^\top \vU_k = \mI_k \}\,.
\end{align*}
It is clear that $g_{\widehat{\vU}_k} \in \cG_k$. Observe that $g_{\widehat{\vU}_k}$ is a random function that depends on the data $\vX$, and we will bound the bias term using the empirical process theory. Recall that $w(\vx)=\rd q(\vx)/\rd p(\vx)$, it is straightforward to verify that 
\begin{align*}
	 \E_{\vx \sim q} \left\| P^\perp_{\widehat{\vU}_k} \left(\vLambda^{1/2} \vphi(\vx) \right) \right\|^2 &=  \int_X \left\| P^\perp_{\widehat{\vU}_k} \left(\vLambda^{1/2} \vphi(\vx)\sqrt{w(\vx)} \right) \right\|^2 \rd p(\vx)\,, \\
	\widehat{\E}_n \left\| P^\perp_{\widehat{\vU}_k} \left(\vLambda^{1/2} \vphi(\vx)\sqrt{w(\vx)} \right) \right\|^2 &= \frac{1}{n} \sum_{i=1}^n \left\| P^\perp_{\widehat{\vU}_k} \left(\vLambda^{1/2} \vphi(\vx_i) \sqrt{w(\vx_i)}\right) \right\|^2 \\
	&= \frac{1}{n} \Tr\left( P^\perp_{\widehat{\vU}_k} \vLambda^{1/2} \vphi(\vX){\vW}({\vX})\vphi^\top(\vX) \vLambda^{1/2}  P^\perp_{\widehat{\vU}_k} \right)\\
 &= \frac{1}{n} \sum_{j > k} \lambda_j(\vK(\vX, \vX){\vW}({\vX}))\,.
\end{align*}
Using symmetrization in Lemma~\ref{lem:symmetrization} with $\kappa W_{\max}$, where $W_{\max}$ is the uniform boundedness of re-weighting ratio given by \cref{assumption:IW_assumption}, with probability at least $1-2\delta$, we have
\begin{align*}
	&  \int_X \left\| P^\perp_{\widehat{\vU}_k} \left(\vLambda^{1/2} \vphi(\vx) \right) \right\|^2 \rd q(\vx) -  \frac{1}{n} \sum_{j > k} \lambda_j( \vK(\vX, \vX){\vW}({\vX}))  \\
	= &\E_p \left\| P^\perp_{\widehat{\vU}_k} \left(\vLambda^{1/2} \vphi(\vx) \sqrt{w(\vx)}\right) \right\|^2  - \widehat{\E}_n \left\| P^\perp_{\widehat{\vU}_k} \left(\vLambda^{1/2} \vphi(\vx)  \sqrt{w(\vx)}\right) \right\|^2 \\
	\leq & \sup_{\vU_k: \vU_k^\top \vU_k = \mI_k} \left( \E - \widehat{\E}_n \right) \left\| P^\perp_{\vU_k} \left(\vLambda^{1/2} \vphi(\vx)  \sqrt{w(\vx)}\right) \right\|^2 \\
	\leq & 2\E_\epsilon \sup_{\vU_k: \vU_k^\top \vU_k = \mI_k} \frac{1}{n} \sum_{i=1}^n \epsilon_i \left( \left\| \vLambda^{1/2} \vphi(\vx_i)  \sqrt{w(\vx_i)} \right\|^2 -  \left\| P_{\vU_k} \left(\vLambda^{1/2} \vphi(\vx_i)  \sqrt{w(\vx_i)}\right) \right\|^2 \right) + 3 \kappa W_{\max}
		 \sqrt{\frac{\log 1/\delta}{2n}} \,,
\end{align*}
by the Pythagorean theorem. Since $\epsilon_i$'s are symmetric and zero-mean and $\left\| \vLambda^{1/2} \vphi(\vx_i)  \right\|^2$ does not depend on $\vU_k$, the last expression is equal to
\begin{align*}
	& 2\E_\epsilon \sup_{g \in \cG_k} \frac{1}{n} \sum_{i=1}^n \epsilon_i g(\vx_i)  + 3 \kappa W_{\max}
		 \sqrt{\frac{\log 1/\delta}{2n}}\,. 
\end{align*}
We further bound the Rademacher complexity of the set $\cG_k$
\begin{align*}
	&\E_\epsilon \sup_{g \in \cG_k} \frac{1}{n}\sum_{i=1}^n \epsilon_i g(\vx_i) = \E_\epsilon \sup_{\vU_k} \frac{1}{n}\sum_{i=1}^n \epsilon_i g_{\vU_k}(\vx_i) \\
	& = \E_\epsilon \frac{1}{n} \sup_{\vU_k} \left\langle \vU_k \vU_k^\top,  \sum_{i=1}^n \epsilon_i w(\vx_i) \vLambda^{1/2} \vphi(\vx_i) \vphi^\top(\vx_i) \vLambda^{1/2}  \right\rangle \\
	& \leq \frac{\sqrt{k}}{n} \E_\epsilon \left\| \sum_{i=1}^n \epsilon_i w(\vx_i) \vLambda^{1/2} \vphi(\vx_i) \vphi^\top(\vx_i) \vLambda^{1/2}  \right\|_F \,,
\end{align*}
by the Cauchy-Schwarz inequality and the fact that $\| \vU_k \vU_k^\top\|_F \leq \sqrt{k}$. The last expression is can be further evaluated by the independence of $\epsilon_i$'s
\begin{align*}	
	\frac{\sqrt{k}}{n} \left\{ \E_\epsilon \left\| \sum_{i=1}^n w(\vx_i) \epsilon_i \vLambda^{1/2} \vphi(\vx_i) \vphi^\top(\vx_i) \vLambda^{1/2}  \right\|_F^2 \right\}^{1/2} & = \frac{\sqrt{k}}{n} \left\{ \sum_{i=1}^n w(\vx_i)^2\left\|  \vLambda^{1/2} \vphi(\vx_i) \vphi^\top(\vx_i) \vLambda^{1/2}  \right\|_F^2  \right\}^{1/2} \\
	& = \sqrt{\frac{k}{n}} \sqrt{ \frac{\sum_{i=1}^n w(\vx_i)^2 K(\vx_i, \vx_i)^2}{n}}\,.
\end{align*}
We have, with probability at least $1 - 2n\delta$,
\begin{align*}
	{\tt (A)}\leq \inf_{0\leq k \leq n}  \left\{  \frac{1}{n} \sum_{j > k} \lambda_j( \vK(\vX, \vX){\vW}({\vX})) + 2 \sqrt{\frac{k}{n}} \sqrt{ \frac{\sum_{i=1}^n w(\vx_i)^2 K(\vx_i, \vx_i)^2}{n}} + 3\kappa W \sqrt{\frac{\log 1/\delta}{2n}} \right\}.
\end{align*}

{\bf Part {\tt (B)}} involves the regularization parameter $\lambda$ and the general weighting function $\overline{w}$. 
Recall the SVD decomposition of $\vLambda^{\frac{1}{2}}\vphi(\vX)=\widehat{\vU}\widehat{\vSigma} \widehat{\vV}^\top$, by direct computation, we have 
\begin{align*}
    &\vLambda^{1/2} \vphi(\vX) [\vphi(\vX)^\top \vLambda \vphi(\vX)]^{-1} \left[I + \vphi^\top(\vX)\vLambda \vphi(\vX)  \overline{\vW}({\vX})/(\lambda n) \right]^{-1} \vphi(\vX)^\top \vLambda^{1/2}\\
    = & \widehat{\vU}[\mI+ \widehat{\vSigma}\widehat{\vV}^\top\overline{\vW}({\vX})\widehat{\vV}\widehat{\vSigma} /(\lambda n)]^{-1} \widehat{\vU}^\top \,.
\end{align*}
It is also straightforward to verify that
\begin{align*}
&\widehat{\E}_n \left\| \vLambda^{1/2} \vphi(\vX) (\vphi(\vX)^\top \vLambda \vphi(\vX))^{-1} \left(I + \vphi^\top(\vX)\vLambda \vphi(\vX)  \overline{\vW}({\vX})/(\lambda n) \right)^{-1} \vphi(\vX)^\top \vLambda^{1/2} \left(\vLambda^{1/2} \vphi(\vx)\sqrt{w(\vx)} \right) \right\|^2 \\
=&  \widehat{\E}_n \left\| \widehat{\vU}(\mI+ \widehat{\vSigma}\widehat{\vV}^\top\overline{\vW}({\vX})\widehat{\vV}\widehat{\vSigma} /(\lambda n))^{-1} \widehat{\vU}^\top\left(\vLambda^{1/2} \vphi(\vx)\sqrt{w(\vx)} \right) \right\|^2\\
=&\frac{1}{n} \Tr\left(
\widehat{\vU}(\mI+\widehat{\vSigma}\widehat{\vV}^\top\overline{\vW}({\vX})\widehat{\vV}\widehat{\vSigma} /(\lambda n))^{-2}\widehat{\vU}^\top \left(\vLambda^{1/2} \vphi(\vX) {\vW}({\vX}) \vphi^\top(\vX) \vLambda^{1/2}\right)
\right)\\
=&\frac{1}{n} \Tr\left(
(\mI+\widehat{\vSigma}\widehat{\vV}^\top\overline{\vW}({\vX})\widehat{\vV}\widehat{\vSigma} /(\lambda n))^{-2} \widehat{\vSigma}\widehat{\vV}^\top {\vW}({\vX})\widehat{\vV}\widehat{\vSigma}
\right)\\
=& \frac{1}{n} \Tr\left((\widehat{\vV}\widehat{\vSigma})^{-1}
(\mI+\widehat{\vV}\widehat{\vSigma}\widehat{\vSigma}\widehat{\vV}^\top\overline{\vW}({\vX}) /(\lambda n))^{-2} (\widehat{\vV}\widehat{\vSigma})\widehat{\vSigma}\widehat{\vV}^\top {\vW}({\vX})\widehat{\vV}\widehat{\vSigma}
\right)\quad \text{[using $(\mI+AB)^{-1}=B^{-1}(\mI+BA)^{-1} B$]} 
\\
=& \lambda^2\Tr\left(\left(\lambda \mI + \frac{\vK(\vX,\vX)\overline{\vW}({\vX})}{n}\right)^{-2} \frac{\vK(\vX,\vX){\vW}({\vX})}{n}\right) \,.
\end{align*}

Therefore, by Lemma~\ref{lem:symmetrization} with $\kappa W$, with probability at least $1 - 2 \delta$, we have
\begin{align*}
   & (\E_p-\widehat{\E}_n) \left\| \widehat{\vU}(\mI+ \widehat{\vSigma}\widehat{\vV}^\top\overline{\vW}({\vX})\widehat{\vV}\widehat{\vSigma} /(\lambda n))^{-1} \widehat{\vU}^\top\left(\vLambda^{1/2} \vphi(\vx)\sqrt{w(\vx)} \right) \right\|^2 \\
	\leq & \sup_{\vU} \left( \E_p - \widehat{\E}_n \right) \left\| {\vU}(\mI+ \widehat{\vSigma}\widehat{\vV}^\top\overline{\vW}({\vX})\widehat{\vV}\widehat{\vSigma} /(\lambda n))^{-1} {\vU}^\top \left(\vLambda^{1/2} \vphi(\vx)  \sqrt{w(\vx)}\right) \right\|^2 \\
	\leq & 2\E_\epsilon \sup_{\vU} \frac{1}{n} \sum_{i=1}^n \epsilon_i\left\| {\vU}(\mI+ \widehat{\vSigma}\widehat{\vV}^\top\overline{\vW}({\vX})\widehat{\vV}\widehat{\vSigma} /(\lambda n))^{-1} {\vU}^\top \left(\vLambda^{1/2} \vphi(\vx_i)  \sqrt{w(\vx_i)}\right) \right\|^2  + 3 \kappa W_{\max}
		 \sqrt{\frac{\log 1/\delta}{2n}} \,.
\end{align*}
Similarly, we obtain 
\begin{align*}
 & \E_\epsilon \sup_{U} \frac{1}{n} \sum_{i=1}^n \epsilon_i\left\| {\vU}(\mI+ \widehat{\vSigma}\widehat{\vV}^\top\overline{\vW}({\vX})\widehat{\vV}\widehat{\vSigma} /(\lambda n))^{-1} {\vU}^\top \left(\vLambda^{1/2} \vphi(\vx_i)  \sqrt{w(\vx_i)}\right) \right\|^2\\
    \leq 	&  \left\|(\mI+ \widehat{\vSigma}\widehat{\vV}^\top\overline{\vW}({\vX})\widehat{\vV}\widehat{\vSigma} /(\lambda n))^{-2}
    \right\|_F\cdot  \frac{1}{n} \E_\epsilon \left\| \sum_{i=1}^n \epsilon_i w(\vx_i) \vLambda^{1/2} \vphi(\vx_i) \vphi^\top(\vx_i) \vLambda^{1/2}  \right\|_F,\\
    \leq & \left\|(\mI+ \widehat{\vSigma}\widehat{\vV}^\top\overline{\vW}({\vX})\widehat{\vV}\widehat{\vSigma} /(\lambda n))^{-2}
    \right\|_F \cdot \sqrt{\frac{1}{n}} \sqrt{ \frac{\sum_{i=1}^n w(\vx_i)^2 K(\vx_i, \vx_i)^2}{n}}\,,
\end{align*}
where
\begin{align*}
     &\left\|(\mI+ \widehat{\vSigma}\widehat{\vV}^\top\overline{\vW}({\vX})\widehat{\vV}\widehat{\vSigma} /(\lambda n))^{-2}
    \right\|_F^2 = \Tr\left((\mI+ \widehat{\vSigma}\widehat{\vV}^\top\overline{\vW}({\vX})\widehat{\vV}\widehat{\vSigma} /(\lambda n))^{-4}
    \right)\\
    =& \Tr\left((\mI+ \vK(\vX,\vX)\overline{\vW}({\vX}) /(\lambda n))^{-4}
    \right) = \lambda^4 \Tr\left(\left(\lambda \mI+ \frac{\vK(\vX,\vX)\overline{\vW}({\vX})}{n} \right)^{-4}
    \right)\leq n \,.
\end{align*}
In total, we have
\begin{align*}
    {\tt (B)}\leq &\lambda^2\left\{\Tr\left(\left(\lambda \mI + \frac{\vK(\vX,\vX)\overline{\vW}({\vX})}{n}\right)^{-2} \frac{\vK(\vX,\vX){\vW}({\vX})}{n}\right)+\sqrt{ \frac{\sum_{i=1}^n w(\vx_i)^2 K(\vx_i, \vx_i)^2}{n}}
    \right\}\\
    +& 3\kappa W_{\max}\sqrt{\frac{\log 1/\delta}{2n}}\,.
\end{align*}

In {\tt (A)}, if we take $k=0$, then with probability $1-4\delta$, 
\begin{align*}
    {\tt (A)+(B)}\leq&  \Tr\left(\frac{\vK(\vX, \vX){\vW}({\vX})}{n} \right) + \lambda^2\left\{\Tr\left(\left(\lambda \mI + \frac{\vK(\vX,\vX)\overline{\vW}({\vX})}{n}\right)^{-2} \frac{\vK(\vX,\vX){\vW}({\vX})}{n}\right)\right. \\
    + & \left. \sqrt{ \frac{\sum_{i=1}^n w(\vx_i)^2 K(\vx_i, \vx_i)^2}{n}}
    \right\} + 6\kappa W_{\max}\sqrt{\frac{\log 1/\delta}{2n}} \,.
\end{align*}

In the next, we consider the discretization of $\vK$ to $\vK^{\rm lin}$, according to \citep[Proposition A.2]{liang2020just}, the kernel matrix admits the following asymmetric approximation with $\theta_p := \frac{1}{2} - \frac{2}{8+m_p}$
\begin{align*}
 		\left\| \vK(\vX, \vX) - \vK^{\rm lin}(\vX, \vX) \right\| &\leq d^{-\theta_p} (\delta^{-1/2} + \log^{\frac{1+\epsilon}{2}} d)\,, \quad \text{w.p.}~ 1-\delta-d^{-2} \,.
\end{align*}
Therefore, we have
\begin{align*}
    \left|\Tr\left(\frac{\vK(\vX, \vX){\vW}({\vX})}{n} \right)-\Tr\left(\frac{\vK^{\rm lin}(\vX, \vX){\vW}({\vX})}{n} \right)\right| \leq W_{\max}\cdot \left\| \vK(\vX, \vX) - \vK^{\rm lin}(\vX, \vX) \right\|\,.
\end{align*}
Besides, we have the following estimates
\begin{align*}
    \left\|
    \left(\lambda \mI + \frac{\vK\overline{\vW}}{n}\right)^{-1} \left(\lambda \mI + \frac{\vK^{\rm lin}\overline{\vW}}{n}\right)
    \right\|&= \left\|
    \left(\lambda \overline{\vW}^{-1} + \frac{\vK}{n}\right)^{-1} \left(\lambda \overline{\vW}^{-1} + \frac{\vK^{\rm lin}}{n}\right)
    \right\|\\
   & \leq 1 +  \left\|
    \left(\lambda\cdot n \overline{\vW}^{-1} + {\vK}\right)^{-1} \left({\vK-\vK^{\rm lin}}\right)
    \right\|\\
   & \leq 1+ \frac{\gamma_p}{\gamma_p - d^{-\theta_p} (\delta^{-1/2} + \log^{\frac{1+\epsilon}{2}} d)} \leq 2\,.
\end{align*}

Then we further have 
\begin{align*}
&\left(\lambda \mI + \frac{\vK\overline{\vW}}{n}\right)^{-2} \frac{\vK{\vW}}{n} \\
    =& \left(\lambda \mI + \frac{\vK^{\rm lin}\overline{\vW}}{n}\right)^{-2} \left(\lambda \mI + \frac{\vK^{\rm lin}\overline{\vW}}{n}\right)^{2}\left(\lambda \mI + \frac{\vK\overline{\vW}}{n}\right)^{-2}\frac{\vK^{\rm lin}{\vW}}{n} + \left(\lambda \mI + \frac{\vK\overline{\vW}}{n}\right)^{-2} \frac{(\vK-\vK^{\rm lin}){\vW}}{n} 
\,.
\end{align*}
Accordingly, we have
\begin{align*}
    \lambda^2\Tr\left(\left(\lambda \mI + \frac{\vK\overline{\vW}}{n}\right)^{-2} \frac{\vK{\vW}}{n}\right) &\leq \lambda^2\left\|
    \left(\lambda \mI + \frac{\vK\overline{\vW}}{n}\right)^{-1} \left(\lambda \mI + \frac{\vK^{\rm lin}\overline{\vW}}{n}\right)
    \right\|^2\Tr\left(\left(\lambda \mI + \frac{\vK^{\rm lin}\overline{\vW}}{n}\right)^{-2} \frac{\vK^{\rm lin}{\vW}}{n}\right) \\
    &+ {\lambda^2}{n} \Tr(\left(\lambda n \mI + {\vK\overline{\vW}}\right)^{-2}) \|{(\vK-\vK^{\rm lin}){\vW}}\|\\
    &\leq 4 \Tr\left(\left(\lambda \mI + \frac{\vK^{\rm lin}\overline{\vW}}{n}\right)^{-2} \frac{\vK^{\rm lin}{\vW}}{n}\right) + d^{-\theta_p} (\delta^{-1/2} + \log^{\frac{1+\epsilon}{2}} d) n^{-1}.
\end{align*}
Therefore, we have, since $n\asymp d$, 
\begin{align*}
    {\tt (A)+(B)}\leq&  \Tr\left(\frac{\vK^{\rm lin}(\vX, \vX){\vW}({\vX})}{n} \right) + 4\lambda^2\Tr\left(\left(\lambda \mI + \frac{\vK^{\rm lin}(\vX,\vX)\overline{\vW}({\vX})}{n}\right)^{-2} \frac{\vK^{\rm lin}(\vX,\vX){\vW}({\vX})}{n}\right) \\
    + &\lambda^2 \kappa W_{\max} + 6\kappa W_{\max}\sqrt{\frac{\log (1/\delta)}{2n}} + 2d^{-\theta_p} (\delta^{-1/2} + \log^{\frac{1+\epsilon}{2}} d)\,.
\end{align*}
Finally, we conclude the proof.
\end{proof}

\subsubsection{Proof of \texorpdfstring{\cref{thm:bias}}{}}

\begin{proof}[Proof of \cref{thm:bias}]
By \citet[Lemma 16]{gogolashvili2023importance} and \cref{assumption:sourcecon}, since $f_\rho\in\mathcal{H}$, under \cref{assumption:sourcecon}, we have the following estimates for ${\sf B}_{\lambda}$, i.e.,
\begin{align*}
    {\sf B}_{\lambda} \leq \lambda^{\overline{r}} \|L_q(L_{\overline{q}}+\lambda)^{-1}\|^{1/2} \|\overline{g}_\rho\|_q,\forall \lambda \geq 0\,.
\end{align*}

The estimation of ${\sf B}_{\rm data}$ relies on \citep[Theorem 20]{gogolashvili2023importance}, under \cref{assumption:sourcecon,assumption:effective_dimension}, we have with probability at least $1-\delta$, 
\begin{align*}
&{\sf B}_{\rm data} \leq 16 \|L_q(L_{\overline{q}}+\lambda)^{-1}\|^{1/2} (\|f_\rho\|_\infty+\|f_\rho\|_\mathcal{H})\cdot \left(\frac{W_{\overline{w}}(d)}{n\sqrt{\lambda}} + \sigma_{\overline{w}}^2(d)\sqrt{\frac{\mathcal{N}_{\overline{q}}^{1-t_{\overline{w}}}(\lambda)}{n\lambda^{t_{\overline{w}}}}}\right)\log\left(\frac{6}{\delta}\right)\\
&\leq 16 \|L_q(L_{\overline{q}}+\lambda)^{-1}\|^{1/2} (\|f_\rho\|_\infty+\|f_\rho\|_\mathcal{H}) ( W_{\overline{w}} d^{c_{\overline{w},1}} n^{-1}\lambda^{-1/2} +  \sigma_{\overline{w}}^2 E_{\overline{q}}^{1-t_{\overline{w}}} d^{2c_{\overline{w},2}} n^{-1/2} \lambda^{-(t_{\overline{w}}+(1-t_{\overline{w}})s_{\overline{q}})/2}))\log(6/\delta)\,,
\end{align*}
given
\begin{align}\label{eq:lambda_suffice}
    n\lambda^{1+t_{\overline{w}}}  \geq 64(W_{\overline{w}}(d) + \sigma_{\overline{w}}^2(d)) (\mathcal{N}_{\overline{q}}(\lambda))^{1-t_{\overline{w}}}\log^2(6/\delta)\,. 
\end{align}
Therefore, for general $\lambda$, we have
\begin{align*}
    {\sf B}\lesssim (\lambda^{\overline{r}}+\lambda^{-\frac{1}{2}}) \|L_q(L_{\overline{q}}+\lambda)^{-1}\|^{\frac{1}{2}}. 
\end{align*}

Recall that $\lambda=C_\lambda^{-c_\lambda}$ and $n\sim d$, we have with probability at least $1-\delta$, 
\begin{align*}
    &{\sf B}_{\rm data} + {\sf B}_{\lambda} \leq \lambda^{\overline{r}} \|L_q(L_{\overline{q}}+\lambda)^{-1}\|^{1/2} \|\overline{g}_\rho\|_q \\
    &+16 \|L_q(L_{\overline{q}}+\lambda)^{-1}\|^{1/2} (\|f_\rho\|_\infty+\|f_\rho\|_\mathcal{H}) ( W_{\overline{w}} d^{c_{\overline{w},1}} n^{-1}\lambda^{-1/2} +  \sigma_{\overline{w}}^2 E_{\overline{q}}^{1-t_{\overline{w}}} d^{2c_{\overline{w},2}} n^{-1/2} \lambda^{-(t_{\overline{w}}+(1-t_{\overline{w}})s_{\overline{q}})/2}))\log(6/\delta) \\
    &\lesssim n^{-\overline{r} c_\lambda} + n^{-(1-c_\lambda/2-c_{\overline{w},1})} + n^{-(1/2-2c_{\overline{w},2} -c_\lambda(t_{\overline{w}}+(1-t_{\overline{w}})s_{\overline{q}})/2 )}\,.
\end{align*}
where the last inequality only considers the dependence on $n$. 
Due to the following fact from \cref{assumption:IW_assumption}, we have
\begin{align*}
    1-c_\lambda/2-c_{\overline{w},1} \geq 1/2 -c_{\overline{w},1} \geq 1/2-2c_{\overline{w},2} -c_\lambda(t_{\overline{w}}+(1-t_{\overline{w}})s_{\overline{q}})/2\,,
\end{align*}
we can conclude that the second term decays faster than the third term. 

We choose $c_\lambda$ to balance the first and the third term, i.e. $\overline{r} c_\lambda = \frac{1-4c_{\overline{w},2}-c_\lambda(t_{\overline{w}}+(1-t_{\overline{w}})s_{\overline{q}})}{2}$, which leads to
\begin{align*}
    c_\lambda = \frac{1-4c_{\overline{w},2}}{2\overline{r}+ t_{\overline{w}}+(1-t_{\overline{w}})s_{\overline{q}}}\,. 
\end{align*}
where $c_{\overline{w},2}<0$ from \cref{assumption:IW_assumption} to ensure $c_\lambda>0$. 
Besides, we have
\begin{align*}
   & 64(W_{\overline{w}}(d) + \sigma_{\overline{w}}^2(d)) (\mathcal{N}_{\overline{q}}(\lambda))^{1-t_{\overline{w}}}\log^2(6/\delta) \leq 64(W_{\overline{w}}\cdot d^{c_{\overline{w},1}} + \sigma_{\overline{w}}^2\cdot d^{2c_{\overline{w},2}}) E_{\overline{q}}^{2(1-t_{\overline{w}})} \lambda^{-s_{\overline{q}}(1-t_{\overline{w}})}\log^2(6/\delta)\\
    \leq & 64(W_{\overline{w}}+\sigma_{\overline{w}}^2) d^{2c_{\overline{w},2}} C_\lambda^{-s_{\overline{q}}(1-t_{\overline{w}})} E_{\overline{q}}^{2(1-t_{\overline{w}})} \cdot n^{c_\lambda s_{\overline{q}}(1-t_{\overline{w}})} \log^2(6/\delta)\,.
\end{align*}
Therefore, for \cref{eq:lambda_suffice}, the constant $C_\lambda$ has to satisfy
\begin{align*}
    64(W_{\overline{w}}+\sigma_{\overline{w}}^2) d^{2c_{\overline{w},2}} C_\lambda^{-s_{\overline{q}}(1-t_{\overline{w}})} E_{\overline{q}}^{2(1-t_{\overline{w}})} \cdot n^{c_\lambda s_{\overline{q}}(1-t_{\overline{w}})} \log^2(6/\delta) \leq n\lambda^{1+t_{\overline{w}}} = C_\lambda^{1+t_{\overline{w}}} n^{1-(1+t_{\overline{w}]}) c_\lambda}\,.
\end{align*}
We expand $c_\lambda$, and using the fact $n/d\to\zeta$, we have that for sufficiently large $d$, $n/d\geq \zeta/2$,
\begin{align*}
     64(W_{\overline{w}}+\sigma_{\overline{w}}^2) E_{\overline{q}}^{2(1-t_{\overline{w}})} (2/\zeta)^{2c_{\overline{w},2}}\log^2(6/\delta) n^{ c_\lambda(s_{\overline{q}}(1-t_{\overline{w}})+t_{\overline{w}}+1)+2c_{\overline{w},2}-1}\leq C_\lambda^{1+t_{\overline{w}}+(1-t_{\overline{w}})s_{\overline{q}}}\,,     \end{align*}
Since $\frac{1}{2}\leq \overline{r}\leq 1$,
\begin{align*}
    c_\lambda(s_{\overline{q}}(1-t_{\overline{w}})+t_{\overline{w}}+1)+2c_{\overline{w},2}-1 = \frac{1-2\overline{r}+2c_{\overline{w},2}(2\overline{r}-2-(t_{\overline{w}}+(1-t_{\overline{w}})s_{\overline{q}}))}{2\overline{r}+ t_{\overline{w}}+(1-t_{\overline{w}})s_{\overline{q}}}\leq 0\,.
\end{align*}
the following constraints would suffice for \cref{eq:lambda_suffice}, 
\begin{align*}
    C_\lambda^{1+t_{\overline{w}}+(1-t_{\overline{w}})s_{\overline{q}}} \geq 64(W_{\overline{w}}+\sigma_{\overline{w}}^2) E_{\overline{q}}^{2(1-t_{\overline{w}})} (2/\zeta)^{2c_{\overline{w},2}}\log^2(6/\delta)\,. 
\end{align*}
Recall the definition of $c_{\mathcal{H}}$ in \cref{assumption:sourcecon}, with probability at least $1-\delta$, we have
\begin{align*}
    {\sf B}\leq {\sf B}_{\rm data}+{\sf B}_{\lambda}\leq n^{-\overline{r}c_\lambda+c_{\mathcal{H}}} \|L_q(L_{\overline{q}}+\lambda)^{-1}\|^{1/2}\left\{ 16 2C_{\mathcal{H}}(W_{\overline{w}}+\sigma_{\overline{w}} E_{\overline{q}}^{1-t_{\overline{w}}})\log(6/\delta) C_\lambda^{-\frac{ t_{\overline{w}}+(1-t_{\overline{w}})s_{\overline{q}}}{2}} + C_\lambda^{\overline{r}}\|\overline{g}_{\rho}\|_q \right\}\,. 
\end{align*}

\end{proof}

\section{Experiments}
\label{sec:exp}
To quantitatively evaluate our derived error bounds for the bias
and variance, we generate a synthetic dataset under
a known $f_\rho$, with different decays of the kernel matrix.

\begin{figure}[!htbp]
    \centering
        \subfloat[\centering variance ($\alpha=0.5$)]{{\includegraphics[width=0.33\columnwidth]{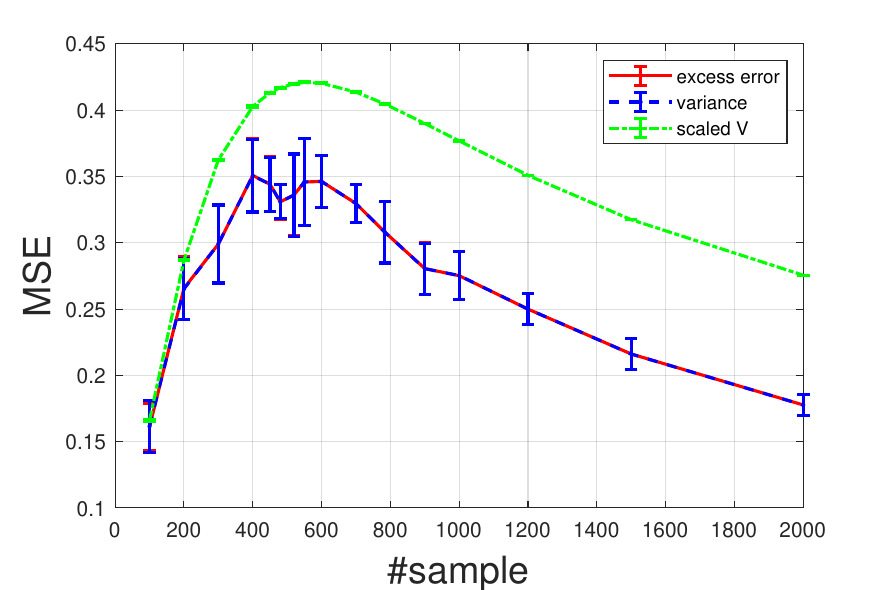}}}%
    \subfloat[\centering variance ($\alpha=1$)]{{\includegraphics[width=0.33\columnwidth]{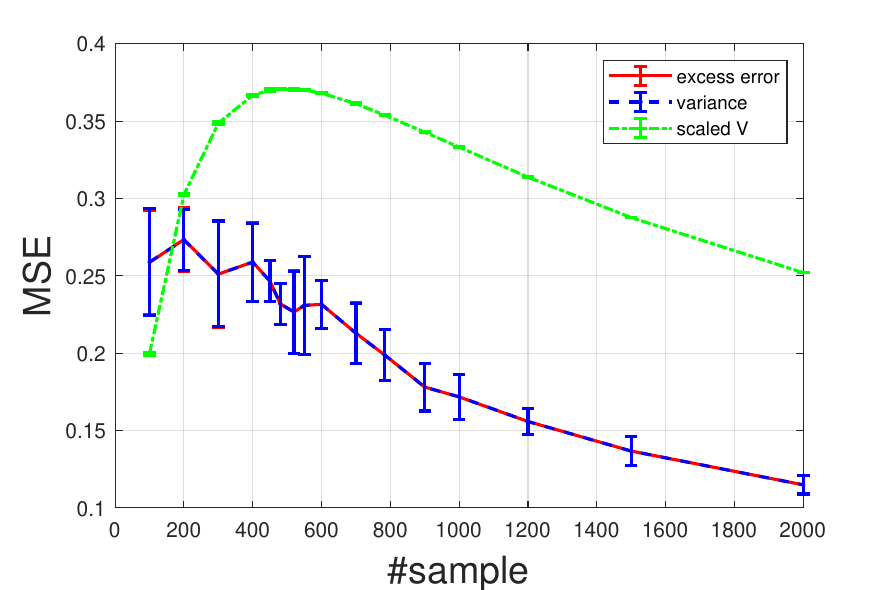}}}%
        \subfloat[\centering variance ($\alpha=1.5$)]{{\includegraphics[width=0.33\columnwidth]{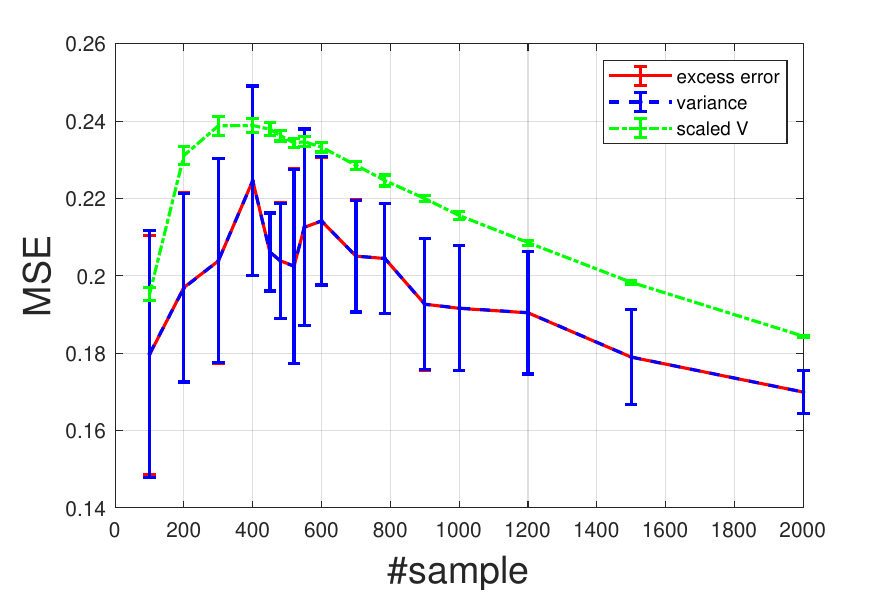}}}%
\\
    \subfloat[\centering bias ($\alpha=0.5$)]{{\includegraphics[width=0.33\columnwidth]{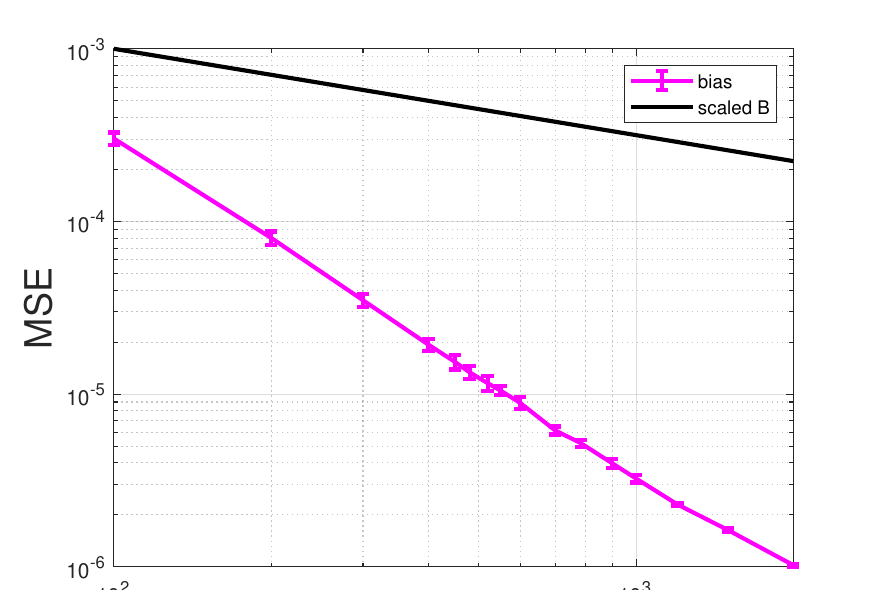}}}%
    \subfloat[\centering bias ($\alpha=1$)]{{\includegraphics[width=0.33\columnwidth]{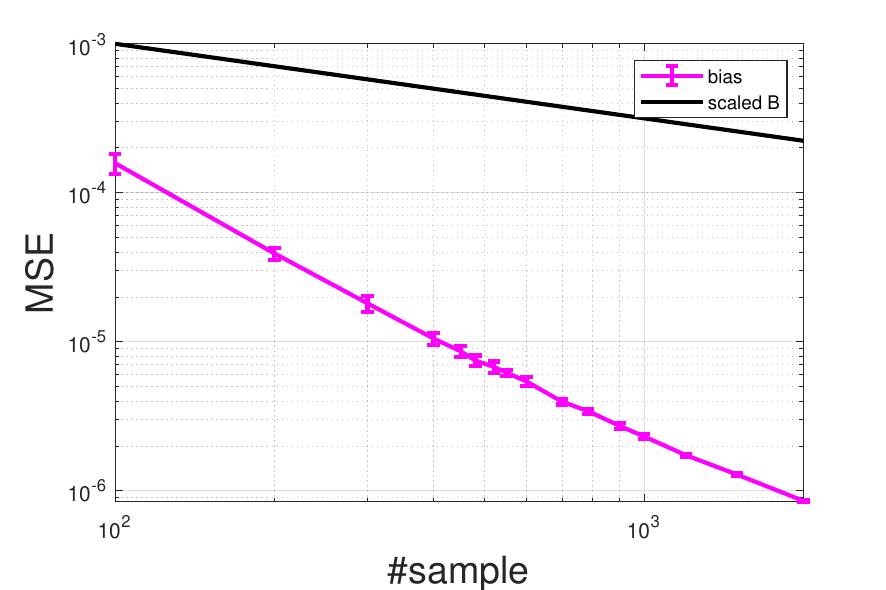}}}%
    \subfloat[\centering bias ($\alpha=1.5$)]{{\includegraphics[width=0.33\columnwidth]{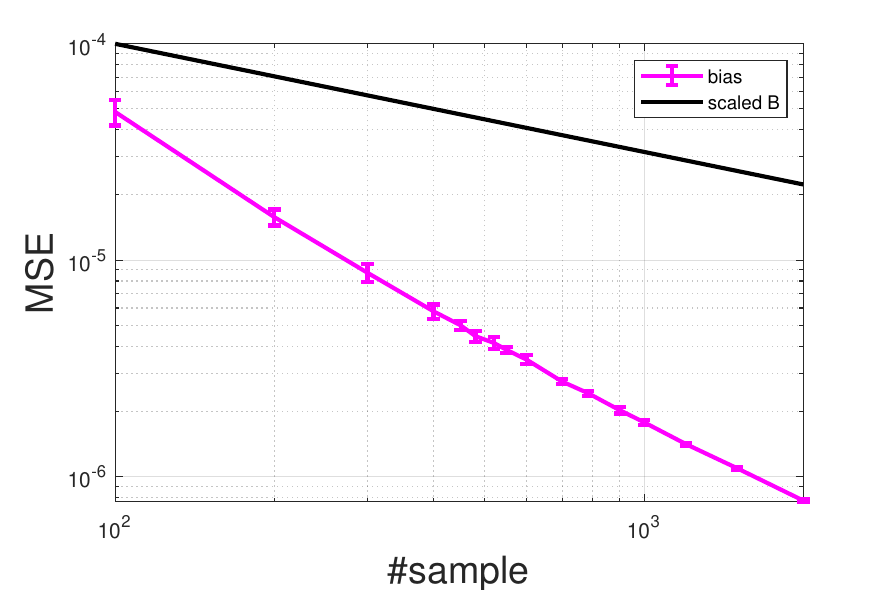}}}%
    \caption{We plot the empirical \texttt{excess error}, \texttt{variance}, \texttt{bias} and the scaled theoretical upper bound \texttt{scaled V} and \texttt{scaled B} under different decays with $\lambda\propto n^{-1/2}$.} 
    \label{fig:exp}
\end{figure}

\paragraph{Eigenvalue decays.} For a positive semi-definite matrix $\vA\in\mathbb{R}^{n\times n}$ with rank $r(\vA)$, we say $\vA$ have one of the following polynomial decay if and only if
 $\lambda_i(\vA) \propto n i^{-a}$ with $a>1$ for $i\leq r(\vA)$. 
 
\paragraph{Data generation.} We assume $y_i=\sin(\|\vx\|^2)+\epsilon$ with the target function $f_\rho(\vx)=\sin(\|\vx\|_2^2)$ and Gaussian noise $\varepsilon$ of zero-mean and unit variance. The training samples $\bm x_i$ are generated from $\vx_{p,i}=\vSigma_p^{1/2}\vz_i$, and the test samples are generated from $\vx_{q,i}=\vSigma_q^{1/2}\vz_i$. Therefore, let $\vX_p$ and $\vX_q$ be the training and test data matrices respectively, and $\vZ=[\vz_1,\cdots,\vz_n]^\top$ we have $\vX_p\vX_p^\top=\vZ \vSigma_p\vZ^\top$ and $\vX_q\vX_q^\top=\vZ \vSigma_q\vZ^\top$. In our experiments, we take 1) $\vSigma_p$ as a diagonal matrix that has diagonal entries with $a=0.5,1,1.5$ for polynomial decay, and $\vSigma_p$ as the perturbed $\vSigma_q$, i.e., $(\vSigma_q)_{i,i}^{-1}=(\vSigma_p)_{i,i}^{-1}+\epsilon',\epsilon'\sim {\rm Unif}[0,1]$; take 2) $\vZ$ as a random orthogonal matrix with almost i.i.d. entries such that $\vX_p\vX_p^\top$ and $\vX_q\vX_q^\top$ have the same eigen-decays as the $\vSigma_p$ and $\vSigma_q$. Specifically, we use the QR decomposition on a random Gaussian matrix to obtain an orthogonal matrix \cite{Yu2016Orthogonal}.

\paragraph{Experimental settings} We set the dimension $d=500$, and the number of test data points to be 2500. We vary the number of training data points as (100, 200, 300, 400, 450, 480, 520, 550, 600, 700, 784, 900, 1000, 1200, 1500, 2000). We set the kernel $K(\vx,\vx') = (1+\langle\vx,\vx'\rangle/d)^p$ with $p=5$, who admits $\beta=p$ independent of $\vSigma_p$. We take the re-weighting function as the truncated probability ratio of distribution $p$ and $q$, i.e., let $\overline{q}=q$ and truncate the ratios to 10. 
Finally, we run on 10 random seeds and calculate the mean and average. 

\paragraph{Choice of $\lambda$} For the target function $f_\rho$ that belongs to the RKHS, we have the source condition $\overline{r}=1/2$. Besides, for the distribution $p$ of the polynomial decay $\alpha$, we take the capacity constant $s_{\overline{q}}=1$. By the boundedness of ratios, we have $t_{\overline{w}}=c_{\overline{w},2}=0$. By \cref{thm:bias}, we have $c_\lambda=1/2$. Therefore, we set $\lambda\propto n^{-1/2}$. 

\paragraph{Observations} 
\cref{fig:exp} (a) - (f) show the trends of the test risk, variance, and bias, which match our upper bound. From the log-log plot of the bias, we observe that our bound is upper bound but not identical to the true rate of the bias decay. 
Besides, if $n$ is large, the upper bound of variance and bias will tend to zero under the IW strategy, which demonstrates that the IW strategy is not harmful to high dimensional kernel methods under covariate shift, at least. All the variances show the unimodal property, and the derived upper bound (as well as the peak) coincides with the empirical ones.

\end{document}